\documentclass[letterpaper]{article}
\usepackage[preprint]{aaai2027}
\usepackage[table,dvipsnames]{xcolor}
\usepackage{adjustbox}
\usepackage{multirow}
\usepackage{makecell}
\usepackage{enumitem}
\usepackage{array}
\usepackage{tabularx}
\newcolumntype{Y}{>{\raggedright\arraybackslash}X}

\newcommand{\std}[1]{\textsubscript{\textcolor{gray}{\(\pm\)#1}}}
\def\redc{\cellcolor[HTML]{FF999A}}
\def\orangec{\cellcolor[HTML]{FFCC99}}
\def\yellowc{\cellcolor[HTML]{FFF8AD}}
\usepackage[most]{tcolorbox}
\newcommand{\ours}{GROM}
\usepackage[hyphens]{url}
\usepackage{graphicx}
\usepackage{amsthm}
\usepackage{natbib}
\usepackage{caption}
\newtheorem{theorem}{Theorem}

\newtheorem{remark}{Remark}
\newtheorem{apxtheorem}{Theorem}[section]
\newtheorem{apxlemma}[apxtheorem]{Lemma}

\usepackage{algorithm}
\usepackage{algorithmic}
\usepackage{graphicx, amsmath, amssymb}
\usepackage{mathtools}
\newcommand{\reqyes}{\cellcolor[HTML]{E8F5E9}\textcolor[HTML]{1B7F3A}{\ensuremath{\checkmark}}}
\newcommand{\reqno}{\cellcolor[HTML]{FCE8E6}\textcolor[HTML]{B3261E}{\ensuremath{\times}}}
\newcommand{\reqhead}[1]{\textbf{\scriptsize #1}}
\usepackage{newfloat}
\usepackage{listings}
\DeclareCaptionStyle{ruled}{labelfont=normalfont,labelsep=colon,strut=off}
\floatstyle{ruled}
\newfloat{listing}{tb}{lst}{}
\floatname{listing}{Listing}

\usepackage{booktabs}
\usepackage{arydshln}
\title{GROM: \underline{G}radient-Free \underline{R}apid \underline{O}ne-Shot \underline{M}achine Unlearning}

\author{
    Paweł Batorski\textsuperscript{\rm 1},
    Przemysław Spurek\textsuperscript{\rm 2,\rm 3},
    Paul Swoboda\textsuperscript{\rm 1}
}
\affiliations{
    \textsuperscript{\rm 1}Heinrich Heine University Düsseldorf\\
    \textsuperscript{\rm 2}Jagiellonian University\\
    \textsuperscript{\rm 3}IDEAS Research Institute
}

\begin{document}

\makeatletter
\let\grom@aaai@maketitle\@maketitle
\def\@maketitle{%
  \grom@aaai@maketitle
  {\centering
  \makebox[\textwidth]{%
    \includegraphics[height=0.37\textwidth]{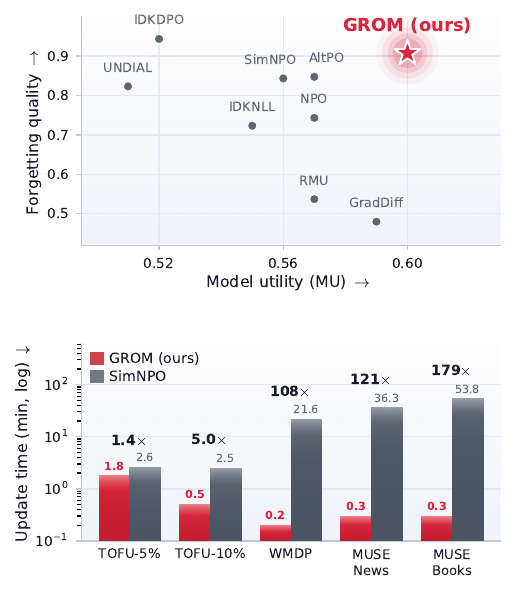}%
    \hfill
    \includegraphics[height=0.37\textwidth]{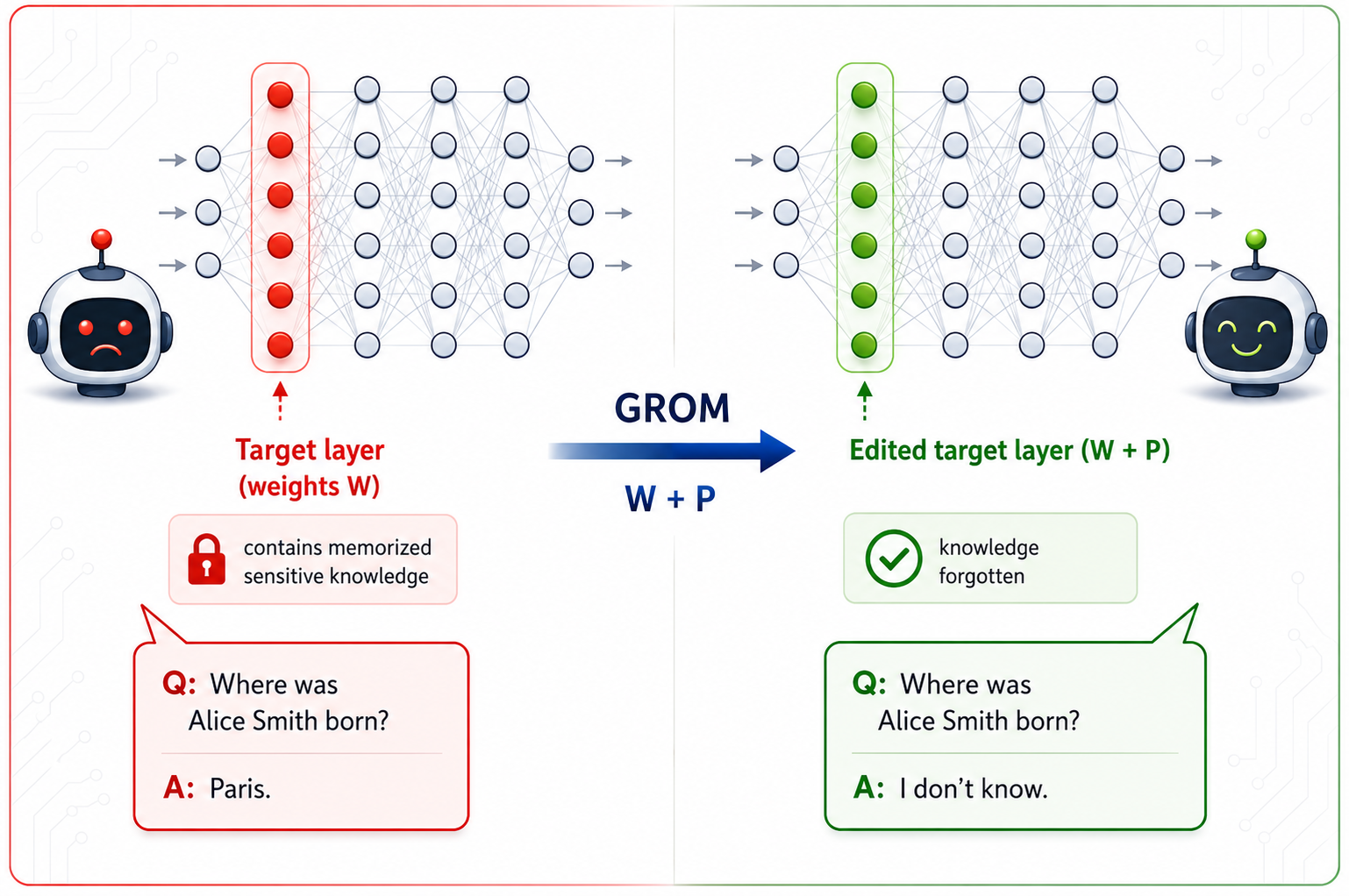}%
  }\par
  \captionof{figure}{\textbf{\ours{} is fast and Pareto-best.} \emph{Left, top:} on TOFU-10\%,
  \ours{} (red star) reaches the top-right corner, strong forgetting and the highest
  model utility. \emph{Left, bottom:} its
  closed-form edit applies in seconds, up to $\sim\!180\times$ faster than SimNPO, a gap that
  widens on the larger 7--8B corpus benchmarks (single NVIDIA H100). \emph{Right:} \ours{}
  replaces iterative fine-tuning with a single gradient-free closed-form edit $W\!+\!P$ to one
  target layer, erasing the targeted knowledge in one step.}
  \label{fig:teaser}
  \par}
}
\makeatother

\maketitle

\begin{abstract}
Machine unlearning has become a critical capability for safely removing specific, sensitive knowledge from large language models (LLMs). Current state-of-the-art approaches primarily rely on iterative, training-time unlearning via fine-tuning. However, even when utilizing parameter-efficient dimensionality reduction techniques like LoRA, gradient-based optimization remains computationally expensive and lacks explicit analytical formulations. It can also leave the targeted knowledge merely hidden rather than removed, to the point that simply quantizing the unlearned model restores much of what it was supposed to have erased. To resolve this, we propose a novel one-shot unlearning approach, abandoning iterative optimization in favor of a direct, exact analytical solution. We frame the unlearning process as a ridge-regularized least-squares optimization problem, deriving a closed-form additive update for targeted weight matrices. This update forces the selected layer to suppress unwanted content while strictly preserving its behavior on retained data. Computed from gradient-free forward passes alone, with no backpropagation and no iteration to convergence, GROM applies the weight edit in mere seconds, which makes it orders of magnitude faster than traditional fine-tuning. Extensive evaluations demonstrate that GROM achieves state-of-the-art forgetting-utility trade-offs on TOFU-5\%, TOFU-10\%, MUSE-Books, MUSE-News and WMDP, significantly reducing computational overhead without sacrificing overall model performance. Because the update removes the targeted content from the weights instead of masking it, GROM also withstands the low-bit quantization attack that recovers much of the content a gradient-based baseline had appeared to forget.
Our code is publicly available at \url{https://github.com/Batorskq/GROM}.

\end{abstract}


\section{Introduction}

Large Language Models (LLMs) frequently memorize sensitive, private, or copyrighted information from their vast training corpora. Consequently, machine unlearning has emerged as a crucial mechanism to selectively erase this targeted knowledge. Existing interventions span a continuum from inference-time mitigation to parameter editing and training-time updates \citep{liu2025rethinking,ren2025sok}. Mitigation strategies, such as decoding controls or prompt-based defenses \citep{yu2021differentially,huang2024offset,thaker2024guardrail}, act as lightweight shields that deflect rather than physically remove the underlying information. Conversely, localized parameter editing methods \citep{ilharco2022editing, meng2022rome} offer rapid updates for specific associations but often struggle to scale effectively to broad, distributional forget sets.

Currently, training-time unlearning provides the strongest performance by optimizing model parameters to reduce the likelihood of the forget set while preserving retain-set utility. This encompasses a wide range of objectives, including gradient ascent \citep{ga, yao2024llmunlearning}, reverse KL divergence \citep{rkld}, and preference-style optimization \citep{dpo, npo, simnpo}. While effective, these methods inherently rely on iterative fine-tuning. Even when utilizing parameter-efficient dimensionality reduction techniques like LoRA, gradient-based optimization remains computationally expensive. More troubling, the forgetting it produces can be superficial: recent work shows that simply quantizing an unlearned model restores much of the content it was supposed to have removed \citep{zhang2025quantfail}, which suggests that fine-tuning often hides the targeted knowledge in low-magnitude weight adjustments that low-bit rounding undoes. Crucially, these approaches also lack an explicit analytical formulation, forcing practitioners to rely on costly, step-by-step gradient descent to approximate an unlearned state.

To address these fundamental inefficiencies, we introduce GROM (Gradient-free Rapid One-shot Machine-unlearning), a new unlearning formulation that abandons iterative optimization in favor of a direct, exact analytical solution. Instead of training additional parameters over multiple epochs, GROM computes a precise, additive update to targeted weight matrices in a single step. We frame the unlearning process as a ridge-regularized least-squares optimization problem, forcing a selected layer to suppress unwanted content while strictly maintaining its output on retained data. By solving this system in closed form, GROM computes the optimal weight update from gradient-free forward passes over the forget and retain data, one pair per edited layer, executing the entire unlearning procedure in mere seconds.

The main contributions of this work are threefold:
\begin{itemize}
    \item \textbf{Closed-Form Unlearning Formulation:} We mathematically frame machine unlearning as an exact optimization problem with a closed-form analytical solution, completely bypassing the computational overhead and optimization instability of iterative fine-tuning.
    \item \textbf{Gradient-Free Efficiency:} GROM computes each layer's optimal additive update from gradient-free forward passes alone, with no backpropagation and no iteration to convergence, which makes it up to two orders of magnitude faster than gradient-based approaches.
    \item \textbf{State-of-the-Art Trade-offs:} GROM achieves Pareto-best forgetting-utility trade-offs across five benchmarks, and it withstands the low-bit quantization attack that restores much of what a strong gradient-based baseline had appeared to forget, staying as forgetful as the gold retrained model.
\end{itemize}

\begin{table}[ht]
\centering
\scriptsize
\setlength{\tabcolsep}{4pt}
\renewcommand{\arraystretch}{0.82}
\begin{adjustbox}{max width=\columnwidth}
\begin{tabular}{lccc}
\toprule
\multicolumn{1}{c}{\textbf{Method}} &
\multicolumn{3}{c}{\textbf{Desirable requirement satisfied}} \\
\cmidrule(lr){2-4}
& \reqhead{Closed-form} & \reqhead{Ref.-free} & \reqhead{No teacher} \\
\midrule
GA~\citep{ga}       & \reqno & \reqyes & \reqyes \\
GradDiff~\citep{yao2024large} & \reqno & \reqyes & \reqyes \\
IDKDPO~\citep{maini2024tofu}   & \reqno & \reqno  & \reqyes \\
IDKNLL~\citep{maini2024tofu}   & \reqno & \reqyes & \reqyes \\
RMU~\citep{wmdp}      & \reqno & \reqyes & \reqyes \\
RKLD~\citep{rkld}     & \reqno & \reqyes & \reqno  \\
NPO~\citep{npo}      & \reqno & \reqno  & \reqyes \\
AltPO~\citep{mekala2025altpo}    & \reqno & \reqno  & \reqyes \\
UNDIAL~\citep{undial}   & \reqno & \reqyes & \reqno  \\
SimNPO~\citep{simnpo}   & \reqno & \reqyes & \reqyes \\
\rowcolor[HTML]{F2F7FF}
\ours{}  & \reqyes & \reqyes & \reqyes \\
\bottomrule
\end{tabular}
\end{adjustbox}
\caption{Qualitative comparison of method requirements, where a check marks a satisfied
property. \emph{Closed-form} means the update is a direct analytical solve, one per edited
layer, rather than iterative gradient optimization, which implies no backpropagation, no
iteration to convergence, and seconds rather than minutes (Time column of
Tables~\ref{tab:tofu}--\ref{tab:muse-unlearning}). \emph{Ref.-free} and \emph{No teacher}
mean no reference model and no sanitized teacher are required.}
\label{tab:qualitative-requirements}
\end{table}

\section{Related Work}
LLM unlearning spans several intervention levels, and recent surveys emphasize
that behavioural suppression, parameter modification, and deletion guarantees
should not be conflated \citep{liu2025rethinking, ren2025sok}. These objectives
descend from exact and certified deletion for smaller models
\citep{cao2015machineunlearning, bourtoule2021sisa, ginart2019deletion,
guo2019certified, izzo2021approximate}, a line motivated by right-to-be-forgotten
provisions \citep{rosen2011right, hoofnagle2019european}. Some methods
avoid direct weight edits and instead alter access to unwanted content through
privacy-oriented fine-tuning, logit or offset steering, guardrails, or
in-context control \citep{yu2021differentially, huang2024offset,
thaker2024guardrail, pawelczyk2023context}. These approaches can be
lightweight, but the forgetting often depends on an external inference-time
mechanism. Another line edits parameters more directly, including task
arithmetic and factual editing methods such as ROME and MEMIT
\citep{ilharco2022editing, meng2022rome, meng2023memit, hase2023localization}. Such edits are fast
and localized, but are often designed for specific associations rather than
distributional forget sets involving many tokens and a competing retain set.
Although \ours{} is similarly localized, it targets this broader forget-retain
setting rather than single-fact replacement. Most empirical LLM unlearning instead optimizes a forget-retain loss by
fine-tuning. Prior objectives include KL and distillation-style retention
\citep{rkld, undial}, refusal or ``I don't know'' targets
\citep{maini2024tofu}, gradient-ascent, gradient-difference, and continual
unlearning \citep{ga, yao2024large, liu2022continual}, and preference-style
variants \citep{dpo, npo, simnpo, mekala2025altpo, jia2024soul, ji2024reversing, chen2023unlearn}. This literature shows that
the target and loss shape the forgetting-utility trade-off, but iterative
optimization is costly and sensitive. \ours{} keeps the target-design view
while replacing fine-tuning with a closed-form one-shot update. Evaluation work further shows that benchmark success need not imply durable
deletion, documenting over-unlearning, reversibility, relearning, and
deployment-specific failures such as quantization sensitivity
\citep{xu2025reversibility, zhang2025quantfail, hu2025jogging,
thaker2025weakbenchmarks, chundawat2024fragile, duan2024membership}. In contrast to
procedures that can be fragile under post-edit compression, \ours{}'s
closed-form update is naturally resilient to quantization attacks.

\begin{figure*}[ht]
\centering
\includegraphics[width=0.92\textwidth]{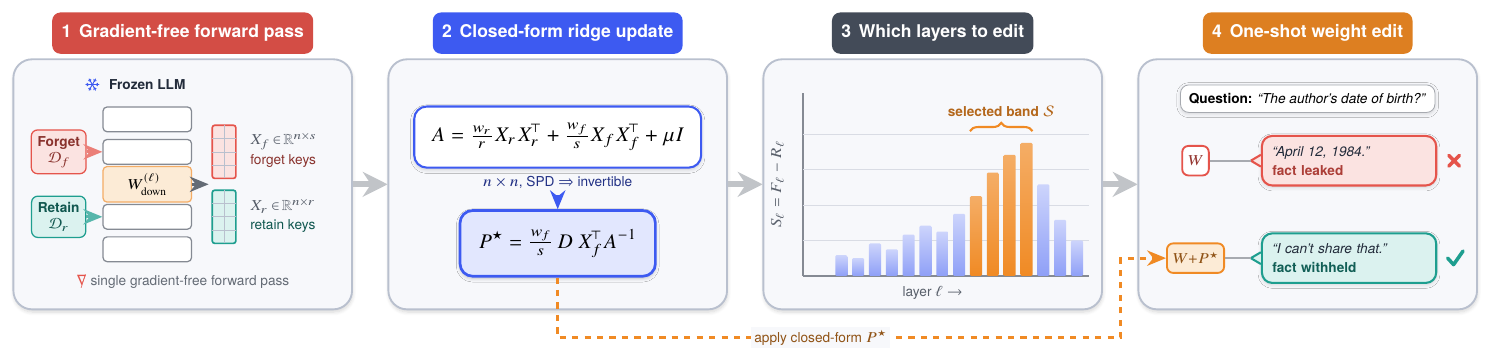}
\caption{Overview of \ours{}. From gradient-free forward passes over the forget
and retain data we collect the per-token keys $X_f,X_r$ (Stage~1). We then solve a single
ridge-regularized least-squares problem in closed form for the additive update $P^\star$
(Stage~2). A logit-lens attribution score $S_\ell=F_\ell-R_\ell$ selects a small late band
of layers $\mathcal{S}$ to edit (Stage~3). Applying $W\!\leftarrow\!W+P^\star$ to those
layers yields the unlearned model, which withholds the memorized content on the same
prompt while leaving retain behaviour intact (Stage~4).}
\label{fig:method}
\end{figure*}

\section{Method}

In this section, we formally define the machine unlearning task as an exact optimization problem. We first establish the problem setup and our ridge-regularized least-squares objective, which admits a unique closed-form solution. We then detail the design of specific unlearning targets based on token suppression and representation corruption. Finally, we describe our logit-lens layer selection strategy and introduce an exact analytical method for auditing the influence of individual deletion requests.

\paragraph{Setup and Editable Matrices.}
Let $W\in\mathbb{R}^{m\times n}$ be a weight matrix we edit, either the language-model head or the down-projection of an MLP block. With one gradient-free forward pass over the forget data and one over the retain data, we collect the \emph{inputs} to $W$, the per-token ``keys'', and stack them column-wise into:
$$ X_f\in\mathbb{R}^{n\times s}\ (\text{forget keys}),\qquad X_r\in\mathbb{R}^{n\times r}\ (\text{retain keys}). $$
A key is taken at every position whose next token should be forgotten (respectively preserved): the answer tokens of a forget/retain QA pair, or, when the forget set is an unlabeled corpus, every token position of that corpus. The layer's current outputs on these keys are $WX_f$ and $WX_r$.

Importantly, adding $\Delta$ to $W$ shifts every output from $Wx$ to $Wx+\Delta x$. We control that shift only if the model uses $Wx$ directly, either by adding it to the residual stream or by reading it as logits. This holds true for the MLP down-projection, the attention output projection, and the LM head, but not for projections whose outputs first pass through a nonlinearity (e.g., up- and gate-projections). Among the usable matrices, we edit the MLP down-projection, since feed-forward blocks are widely reported to store facts and push them toward particular output tokens \citep{geva2021kv, geva2022promoting, dai2022knowledge}. Table~\ref{tab:ablations} checks the alternatives empirically.

\paragraph{Objective and Closed-Form Solution.}
We seek a single additive update $P\in\mathbb{R}^{m\times n}$, applied in closed form ($W\leftarrow W+P$), that simultaneously preserves the retain behavior, $PX_r\approx 0$ (so $(W+P)X_r\approx WX_r$), and steers the forget outputs by a prescribed \emph{unlearning target} $D\in\mathbb{R}^{m\times s}$ that removes the memorized content, $(W+P)X_f\approx WX_f + D$. Here, $D$ specifies the desired output-space displacement for each forget feature.

We frame this trade-off as a ridge-regularized least-squares problem:
$$ \min_{P\in\mathbb{R}^{m\times n}}\; \frac{w_r}{r}\,\|PX_r\|_F^2 +\frac{w_f}{s}\,\|PX_f-D\|_F^2 +\mu\,\|P\|_F^2 $$
where $w_r,w_f>0$ weight retain preservation against forget steering, and $\mu>0$ penalizes the edit size. The target $D$ is any fixed matrix in $\mathbb{R}^{m\times s}$, chosen before solving for $P$. The optimizer of this problem is unique and given in closed form, avoiding iterative approximation entirely.

\begin{theorem}
\label{thm:closed-form}
Let $X_r\in\mathbb{R}^{n\times r}$, $X_f\in\mathbb{R}^{n\times s}$, $D\in\mathbb{R}^{m\times s}$, and let $w_r,w_f,\mu>0$. Define
$$ A = \frac{w_r}{r}X_rX_r^\top + \frac{w_f}{s}X_fX_f^\top + \mu I_n \in\mathbb{R}^{n\times n}. $$
Then the ridge-regularized objective is minimized by the unique matrix
$$ P^\star = \frac{w_f}{s}\, D X_f^\top A^{-1}. $$
\end{theorem}
\noindent
The full proof of Theorem~\ref{thm:closed-form} is provided in Appendix~\ref{app:proofs}.

\paragraph{Unlearning Targets.}
Theorem~\ref{thm:closed-form} holds for any fixed $D$, so the target is where we encode \emph{what} to forget. We use one of two forms, chosen according to how the benchmark probes the forgotten knowledge.

\emph{(i) Token suppression}, used when forgetting is measured through the content the model \emph{generates} (TOFU, MUSE). For each forget position $j$ with gold next token $g_j\in\{1,\dots,|\mathcal{V}|\}$, we steer the output away from that token:
$$ d_j=-\beta\,\alpha_j\,u_j,\qquad u_j= \begin{cases} e_{g_j}, & W=W_{\mathrm{head}},\\[2mm] \frac{w_{g_j}}{\|w_{g_j}\|_2}, & W=W_{\mathrm{down}}^{(\ell)}, \end{cases} $$
so that the post-edit output $(W+P)x^{(f)}_j\approx Wx^{(f)}_j-\beta\alpha_j u_j$ assigns a lower logit to $g_j$. Here $\beta>0$ is the edit strength, $e_{g_j}$ is the one-hot vector for token $g_j$, and $w_{g_j}=W_{\mathrm{head}}[g_j,:]$ is the LM-head row for that token. For a hidden-layer edit, this row is the direction in residual space that most directly increases the logit of $g_j$. The \emph{specificity weight} $\alpha_j\in[0,1]$ confines suppression to tokens distinctive of the forget set:
$$ \alpha_j=\max\!\Big(0,\;1-\frac{\mathrm{rf}(g_j)}{\mathrm{ff}(g_j)}\Big), $$
where, for any token $v$,
\[
\begin{aligned}
\mathrm{ff}(v)
&=
\frac{\sum_{t\in\mathcal{F}}\mathbf{1}[t=v]}
     {\sum_{v'\in\mathcal{V}}\sum_{t\in\mathcal{F}}\mathbf{1}[t=v']},\\
\mathrm{rf}(v)
&=
\frac{\sum_{t\in\mathcal{R}}\mathbf{1}[t=v]}
     {\sum_{v'\in\mathcal{V}}\sum_{t\in\mathcal{R}}\mathbf{1}[t=v']}.
\end{aligned}
\]
Here $\mathcal{F}$ is the list of all forget tokens, with duplicates, and
$\mathcal{R}$ is the analogous retain-token list. Tokens common to both receive
$\alpha_j\!\approx\!0$, whereas forget-specific tokens receive $\alpha_j\!\approx\!1$.

\emph{(ii) Representation corruption}, used when forgetting is measured by multiple-choice accuracy (WMDP). There, the answer is a choice label rather than memorized text, so lowering content-token logits has little effect. Instead, we corrupt the representation itself. With a single fixed random unit vector $u\in\mathbb{R}^m$, we set:
$$ d_j=c\,u\qquad\text{for all } j, $$
which, through the same objective with the retain keys anchored to $0$, drives the layer's output on forget-like inputs toward a fixed, meaningless direction while leaving retain inputs intact. In both cases, the columns are assembled into $D=[\,d_1,\dots,d_s\,]$, fixed before solving for $P$.

\paragraph{Choice of Layers to be Updated.}
Rather than choosing the edited layers only by grid search, we use a logit-lens attribution score \citep{nostalgebraist2020logitlens, belrose2023tunedlens} to identify layers that directly write the memorized forget tokens. For each MLP layer $\ell$, let $o^{(f)}_{\ell,j}\in\mathbb{R}^{d_{\mathrm{model}}}$ be the output of its down-projection at forget position $j$, and let $w_{g_j}\in\mathbb{R}^{d_{\mathrm{model}}}$ be the LM-head row corresponding to the gold next token $g_j$. The direct contribution of layer $\ell$ to the logit of token $g_j$ is $\langle w_{g_j}, o^{(f)}_{\ell,j}\rangle$. We compute the forget and retain effects:
$$ F_\ell = \frac{\sum_{j=1}^{s}\alpha_j \langle w_{g_j}, o^{(f)}_{\ell,j}\rangle} {\sum_{j=1}^{s}\alpha_j}, \qquad R_\ell = \frac{1}{r}\sum_{i=1}^{r} \langle w_{h_i}, o^{(r)}_{\ell,i}\rangle, $$
where $h_i$ is the retain gold next-token id. We then score each layer by $S_\ell = F_\ell - R_\ell$. A large positive $S_\ell$ indicates that layer $\ell$ contributes more to forget-specific gold-token logits than to retain-token logits. 

We treat the number of edited layers as a small edit-width hyperparameter $k$. For a fixed $k$ over candidate layers $\mathcal{L}_{\mathrm{cand}}$, we choose the contiguous window with the largest average score:
$$ \mathcal{S}_k = \arg\max_{\{a,\ldots,a+k-1\}\subset\mathcal{L}_{\mathrm{cand}}} \frac{1}{k}\sum_{\ell=a}^{a+k-1} S_\ell . $$
For each $\ell\in\mathcal{S}_k$, we compute the closed-form update for $W_{\mathrm{down}}^{(\ell)}$, applying the updates sequentially and recomputing features after each edited layer. Figure~\ref{fig:method} summarizes the full pipeline.

\paragraph{Influence of a Single Forget Example.}
Because $P^\star$ is an explicit function of the forget data, we can further ask how much any \emph{single} forget example contributed to the edit, and answer it exactly. We use this measure to \emph{audit} individual deletion requests. A provider asked to account for a user's data must be able to state what that example contributed to the released model.

Partition the forget keys and target by example, $X_f=[\,C_1\,\cdots\,C_N\,]$ and $D=[\,D_1\,\cdots\,D_N\,]$, where $C_i\in\mathbb{R}^{n\times c_i}$ collects the token positions of example $i$, and let $P^\star_{-i}$ be the update recomputed with example $i$ removed (holding the per-token forget weight $\tfrac{w_f}{s}$ fixed). Deleting example $i$ perturbs the Gram matrix $A$ by a symmetric rank-$c_i$ downdate. The Sherman--Morrison--Woodbury identity \citep{hager1989updating} resolves this analytically.

\begin{theorem}
\label{thm:influence}
Let $A_{-i}=A-\tfrac{w_f}{s}C_iC_i^\top$. Then $A_{-i}$ is symmetric positive definite, and the Sherman--Morrison--Woodbury identity yields the deletion influence of example $i$ in closed form as:
$$ \Delta P_i:=P^\star-P^\star_{-i}=L_iR_i^\top,\qquad R_i=A^{-1}C_i, $$
$$ L_i=\tfrac{w_f}{s}\big(D_i-B_{-i}\,C_i\,M_i^{-1}\big),\quad M_i=I_{c_i}-\tfrac{w_f}{s}C_i^\top A^{-1}C_i, $$
where $B_{-i}=P^\star-\tfrac{w_f}{s}D_iC_i^\top A^{-1}$. Hence $\operatorname{rank}\Delta P_i\le c_i$, and $\Delta P_i$ is obtained from the already-computed $A^{-1}$ by a single $c_i\times c_i$ solve, with no $n\times n$ reinversion and no retraining.
\end{theorem}
\noindent
While the same quantity could theoretically be obtained by completely recomputing $P^\star$ without example $i$, Theorem~\ref{thm:influence} ensures computational feasibility. A naive approach would require reforming and reinverting the full $n\times n$ matrix $A$ for every example. The analytical downdate reduces this to a $c_i\times c_i$ solve against a single cached $A^{-1}$, bringing the audit time down to seconds. Gradient-based unlearning admits no analogue and must approximate this with influence functions \citep{koh2017influence} or completely retrain the model per deleted example.

\section{Experiments}

\paragraph{Baselines.}
We compare \ours{} with representative unlearning methods spanning gradient-based,
representation-based, preference-based, and distillation-based approaches.
\textbf{GA}~\citep{ga, maini2024tofu} maximizes the loss on forget examples to reduce
their likelihood, and \textbf{GradDiff}~\citep{yao2024large, maini2024tofu,
liu2022continual} combines that ascent with descent on retain data.
\textbf{TaskVector}~\citep{ilharco2022editing} edits behavior through task-arithmetic
updates in weight space, and \textbf{RMU}~\citep{wmdp} redirects forget representations
toward a fixed random vector while regularizing retain representations.
\textbf{IDKDPO} and \textbf{IDKNLL}~\citep{maini2024tofu} train the model to answer
forget prompts with ``I don't know'' using DPO or NLL, while \textbf{RKLD}~\citep{rkld}
distills from a privacy-sanitized teacher with reverse KL and
\textbf{UNDIAL}~\citep{undial} self-distills with adjusted logits on forget data. Among
preference-style objectives, \textbf{NPO}~\citep{npo} suppresses undesirable forget
responses with a negative-preference loss, \textbf{AltPO}~\citep{mekala2025altpo}
contrasts original against alternate forget answers, and
\textbf{SimNPO}~\citep{simnpo} simplifies NPO by removing the reference model. Because
\ours{} is itself a weight edit, we additionally compare against four locate-then-edit
knowledge editors: \textbf{ROME} and \textbf{MEMIT}~\citep{meng2022rome, meng2023memit},
\textbf{AlphaEdit}~\citep{fang2024alphaedit}, which confines the update to the null space
of the preserved-key covariance, and \textbf{ZeroUnlearn}~\citep{lin2026zerounlearn}, which
remaps forget keys to a neutral state through a multiplicative edit constrained to the null
space of the original forget outputs. All four are given the same per-fact value
optimization and the same retain covariance as \ours{}, so the rows differ only in the
update rule. We compare these editors only on TOFU and ZsRE, where examples can be cast as discrete
facts. The editors require forget data in annotated (subject, relation, object) triples:
they locate a key at the subject and redirect a specific target object. We do not include
them on MUSE or WMDP because those benchmarks provide unstructured corpora or
hazardous-domain question data rather than such triples. \ours{} does not require this structure, since
it can collect keys directly from the forget text.

\paragraph{Experimental Setup} We evaluate \ours{} against established unlearning baselines on TOFU-5\% and TOFU-10\% \citep{maini2024tofu}, MUSE \citep{shi2024muse} and WMDP \citep{wmdp}, scoring all checkpoints with the open-unlearning protocol \citep{dorna2025openunlearning, jin2024rwku}. Target models, per-benchmark metrics, the timing protocol, and the baseline training budgets are given in Appendix~\ref{app:summary}. Throughout, \textbf{Time~(m)} is the wall-clock of the unlearning update on a single NVIDIA H100, excluding model loading and evaluation. It varies across benchmarks because the cost of \ours{} is set by model size and the number of edited layers rather than by forget-set size, the keys being subsampled to a fixed budget, whereas gradient-based baselines scale with epochs times corpus size, so the gap is narrowest on TOFU-5\% and widest on MUSE Books.

\begin{table}[t]
\centering
\scriptsize
\setlength{\tabcolsep}{4pt}
\renewcommand{\arraystretch}{0.82}
\begin{adjustbox}{max width=\columnwidth}
\begin{tabular}{@{}l ccc c !{\vrule} cc@{}}
\toprule
& \multicolumn{3}{c}{Unlearning Efficacy} & Utility & \multicolumn{2}{c}{Summary} \\
\cmidrule(lr){2-4}\cmidrule(lr){5-5}\cmidrule(l){6-7}
Method
& \makecell[c]{1-Rouge-L\\($\uparrow$)}
& \makecell[c]{1-Prob.\\($\uparrow$)}
& \makecell[c]{1-Extr.\\($\uparrow$)}
& \makecell[c]{MU\\($\uparrow$)}
& \makecell[c]{Final\\Score ($\uparrow$)}
& \makecell[c]{Time (m)\\($\downarrow$)} \\
\midrule
\multicolumn{7}{c}{\textbf{TOFU-5\% (LLaMA2-7B-Chat)}} \\
\midrule
Original & 0.04 & 0.01 & 0.05 & 0.62 & 0.33 & --- \\
Retain   & 0.61 & 0.85 & 0.93 & 0.62 & 0.71 & --- \\
\midrule
GradDiff & \redc 1.00 & \redc 1.00 & \orangec 0.96 & 0.56 & \orangec 0.77 & 2.4 \\
IDKDPO   & \orangec 0.98 & 0.40 & 0.85 & 0.57 & 0.66 & 3.3 \\
RKLD     & 0.69 & \yellowc 0.96 & 0.92 & 0.56 & 0.71 & 2.9 \\
NPO      & 0.73 & 0.94 & 0.90 & 0.57 & 0.71 & 2.9 \\
SimNPO   & 0.74 & \orangec 0.97 & 0.92 & \yellowc 0.58 & \yellowc 0.73 & 2.6 \\
ROME     & 0.82 & 0.69 & 0.90 & 0.57 & 0.68 & \redc 0.6 \\
MEMIT    & 0.73 & 0.69 & 0.89 & 0.53 & 0.65 & \orangec 0.7 \\
AlphaEdit   & 0.94 & 0.23 & 0.08 & \orangec 0.61 & 0.52 & \redc 0.6 \\
ZeroUnlearn & 0.80 & 0.80 & \yellowc 0.93 & 0.51 & 0.68 & \redc 0.6 \\
\hline
\rowcolor[HTML]{F2F7FF}
\ours{}  & \yellowc 0.95 & \redc 1.00 & \redc 0.97 & \redc 0.62 & \redc 0.79 & \yellowc 1.8 \\
\midrule
\multicolumn{7}{c}{\textbf{TOFU-10\% (LLaMA3.2-1B-Instruct)}} \\
\midrule
Original & 0.18 & 0.12 & 0.29 & 0.60 & 0.40 & --- \\
Retain   & 0.62 & 0.88 & 0.94 & 0.59 & 0.70 & --- \\
\midrule
RMU      & 0.50 & 0.39 & 0.72 & \yellowc 0.57 & 0.55 & \yellowc 0.6 \\
AltPO    & 0.66 & \yellowc 0.93 & \orangec 0.95 & \yellowc 0.57 & \yellowc 0.71 & 6.9 \\
GradDiff & 0.42 & 0.35 & 0.67 & \orangec 0.59 & 0.53 & 0.8 \\
IDKDPO   & 0.87 & \redc 1.00 & \redc 0.96 & 0.52 & \orangec 0.73 & 6.9 \\
IDKNLL   & \redc 0.98 & 0.45 & 0.74 & 0.55 & 0.64 & 0.8 \\
UNDIAL   & 0.69 & 0.82 & \redc 0.96 & 0.51 & 0.67 & 0.9 \\
NPO      & 0.61 & 0.71 & 0.91 & \yellowc 0.57 & 0.66 & 3.7 \\
SimNPO   & 0.65 & \orangec 0.94 & \yellowc 0.94 & 0.56 & 0.70 & 2.5 \\
ROME     & \redc 0.98 & 0.37 & 0.60 & 0.54 & 0.60 & \redc 0.4 \\
MEMIT    & \yellowc 0.88 & 0.38 & 0.61 & 0.49 & 0.56 & \orangec 0.5 \\
AlphaEdit   & \redc 0.98 & 0.56 & 0.81 & 0.50 & 0.64 & \orangec 0.5 \\
ZeroUnlearn & \orangec 0.92 & 0.47 & 0.81 & 0.44 & 0.59 & \redc 0.4 \\
\hline
\rowcolor[HTML]{F2F7FF}
\ours{}  & 0.78 & \redc 1.00 & \yellowc 0.94 & \redc 0.60 & \redc 0.75 & \orangec 0.5 \\
\bottomrule
\end{tabular}
\end{adjustbox}
\caption{Results on TOFU-5\% and TOFU-10\%. Colors indicate rank among unlearning methods
(red: best, orange: second, yellow: third).}
\label{tab:tofu}
\end{table}

\paragraph{TOFU.}
TOFU~\citep{maini2024tofu} evaluates selective forgetting in question answering about
fictitious authors, where TOFU-$x\%$ designates $x\%$ of the authors for removal while the
model must preserve its behaviour on the remaining ones. Side effects on unrelated
knowledge are measured through the Real Authors and World Facts subsets. We report
forgetting with ROUGE-L~\citep{lin2004rouge}, answer probability, and extraction
strength~\citep{carlini2021extracting, carlini2022quantifying}, and utility with the TOFU model-utility (MU)
score, following~\citep{simnpo}.

On TOFU-5\% (Table~\ref{tab:tofu}) \ours{} attains near-saturated forgetting while
keeping utility close to the original model, giving the best Final Score and the highest
MU. The harder TOFU-10\% setting (same table) splits the baselines into two
failure modes, over-forgetting at the cost of utility (IDKDPO, UNDIAL) or preserving
utility while under-forgetting (RMU). \ours{} avoids both with a single closed-form
update. Metric definitions and edit details are given in Appendix~\ref{app:tofu}.

\paragraph{ZsRE.}
The comparison above places the knowledge editors on an unlearning benchmark, so we also
run the reverse test, evaluating \ours{} inside the harness of~\citep{lin2026zerounlearn}
on ZsRE~\citep{levy2017zeroshot}, a few-shot fact-unlearning benchmark on which those
editors are tuned. Table~\ref{tab:zsre} shows that \ours{} outperforms these tuned
editors: it removes targeted answers more effectively, generalizes better to paraphrases,
and preserves locality at least as well. Metric definitions and our reproduction of their
unedited-model row are given in Appendix~\ref{app:zsre}.

\begin{table}[t]
\centering
\scriptsize
\setlength{\tabcolsep}{4pt}
\renewcommand{\arraystretch}{0.82}
\begin{adjustbox}{max width=\columnwidth}
\begin{tabular}{@{}l cc !{\vrule} c@{}}
\toprule
& \multicolumn{2}{c}{Unlearning Efficacy} & Utility Preservation \\
\cmidrule(lr){2-3}\cmidrule(l){4-4}
Method
& \makecell[c]{Efficacy\\($\downarrow$)}
& \makecell[c]{Generalization\\($\downarrow$)}
& \makecell[c]{Specificity\\($\uparrow$)} \\
\midrule
\multicolumn{4}{c}{\textbf{LLaMA3.2-3B-Instruct}} \\
\midrule
Original    & 32.82\std{4.09} & 32.23\std{4.16} & 28.12\std{2.65} \\
\midrule
ROME        & 32.80\std{4.20} & 32.17\std{4.09} & \redc 28.05\std{2.66} \\
MEMIT       & 32.32\std{4.04} & 31.17\std{4.61} & \yellowc 28.01\std{2.60} \\
AlphaEdit   & \yellowc 29.59\std{3.95} & \yellowc 29.90\std{4.67} & 27.80\std{2.77} \\
ZeroUnlearn & \orangec 27.85\std{3.87} & \orangec 27.52\std{3.87} & 27.73\std{2.70} \\
\hline
\rowcolor[HTML]{F2F7FF}
\ours{}     & \redc 3.52\std{1.39} & \redc 4.15\std{1.32} & \orangec 28.04\std{2.73} \\
\midrule
\multicolumn{4}{c}{\textbf{LLaMA3.1-8B-Instruct}} \\
\midrule
Original    & 40.42\std{4.92} & 36.84\std{4.24} & 29.87\std{2.30} \\
\midrule
ROME        & 40.46\std{4.85} & 36.84\std{4.16} & \yellowc 29.99\std{2.37} \\
MEMIT       & 35.15\std{3.99} & 34.60\std{3.15} & \orangec 30.05\std{2.46} \\
AlphaEdit   & \yellowc 34.12\std{4.16} & \yellowc 34.19\std{3.33} & 29.93\std{2.49} \\
ZeroUnlearn & \orangec 32.67\std{3.43} & \orangec 32.39\std{3.34} & 29.67\std{2.36} \\
\hline
\rowcolor[HTML]{F2F7FF}
\ours{}     & \redc 4.06\std{2.65} & \redc 3.83\std{2.35} & \redc 31.25\std{3.06} \\
\bottomrule
\end{tabular}
\end{adjustbox}
\caption{ZsRE, evaluated in the harness of~\citep{lin2026zerounlearn}, fifty unlearned
facts, mean$\pm$std. }
\label{tab:zsre}
\end{table}

\paragraph{WMDP.}
WMDP~\citep{wmdp} targets the suppression of hazardous knowledge rather than the removal
of memorized text. Following~\citep{simnpo}, we report $1-\mathrm{AccBio}$ on WMDP-Bio as
the forgetting metric, where larger values indicate stronger suppression of the hazardous
slice, and MMLU accuracy~\citep{hendrycks2020measuring} as the utility metric. Several
baselines in Table~\ref{tab:wmdp} reach strong suppression only at the cost of general
ability, the clearest case being GradDiff, which collapses MMLU. \ours{} instead retains
by far the highest general-knowledge utility at competitive suppression, giving the best
Final Score, and it does so with a single closed-form edit at one MLP layer. Details are
given in Appendix~\ref{app:wmdp}.

\begin{table}[t]
\centering
\scriptsize
\setlength{\tabcolsep}{4pt}
\renewcommand{\arraystretch}{0.82}

\begin{adjustbox}{max width=\columnwidth}
\begin{tabular}{@{}lcc!{\vrule}cc@{}}
\toprule
& Unlearning Efficacy & Utility Preservation & \multicolumn{2}{c}{Summary} \\
\cmidrule(lr){2-2}\cmidrule(lr){3-3}\cmidrule(l){4-5}
Method
& \makecell[c]{$1-\mathrm{AccBio}$\\($\uparrow$)}
& \makecell[c]{MMLU\\($\uparrow$)}
& \makecell[c]{Final Score\\($\uparrow$)}
& \makecell[c]{Time (m)\\($\downarrow$)} \\
\midrule

Original & 0.27 & 0.65 & 0.46 & --- \\
UnDIAL   & 0.65 & 0.45 & 0.55 & 20.9 \\
GradDiff & \orangec 0.73 & 0.26 & 0.50 & 20.2 \\
IDKNLL   & 0.66 & \yellowc 0.47 & 0.57 & 20.2 \\
IDKDPO   & 0.67 & 0.44 & 0.56 & 23.9 \\
NPO      & \orangec 0.73 & \orangec 0.51 & \orangec 0.62 & 23.9 \\
SimNPO   & \redc 0.75 & 0.44 & \yellowc 0.60 & 21.6 \\
\midrule
\ours{}  & \yellowc 0.71 & \redc 0.55 & \redc 0.63 & \redc \textbf{0.2} \\

\bottomrule
\end{tabular}
\end{adjustbox}

\caption{Results on WMDP-Bio (Llama-3-8B-Instruct). Forgetting is measured by
$1-\mathrm{AccBio}$ on WMDP-Bio, and utility is measured by MMLU accuracy. Final Score is
$\tfrac12\big((1-\mathrm{AccBio})+\mathrm{MMLU}\big)$.}
\label{tab:wmdp}
\end{table}

\paragraph{MUSE.}
MUSE~\citep{shi2024muse} addresses long-form unlearning, where a model may leak verbatim
passages as well as the underlying facts. MUSE News splits BBC articles into disjoint
forget, retain, and holdout sets. MUSE Books is the harder entangled scenario, using the
Harry Potter novels \citep{eldan2023s, wei2024evaluating} as the forget set and a related fan wiki as the retain set, so the
model must stop reproducing copyrighted text while still answering questions about closely
related permissible material. We report VerbMem (verbatim regurgitation), KnowMem
(forget-corpus knowledge), and PrivLeak (membership leakage against the holdout set), with
retain KnowMem as the utility measure, all defined in Appendix~\ref{app:muse}. In
Table~\ref{tab:muse-unlearning}, \ours{} attains the best Final Score on both corpora,
driving privacy leakage close to zero on News and, on Books, preserving the highest retain
utility while nearly eliminating verbatim memorization, where several baselines erase
forget-set memorization only by destroying retain utility.

\begin{table}[t]
\centering
\scriptsize
\setlength{\tabcolsep}{4pt}
\renewcommand{\arraystretch}{0.82}

\begin{adjustbox}{max width=\columnwidth}
\begin{tabular}{lccc!{\vrule}ccc}
\toprule
& \multicolumn{3}{c}{Unlearning Efficacy} & Utility & \multicolumn{2}{c}{Summary} \\
\cmidrule(r){2-4}\cmidrule(lr){5-5}\cmidrule(l){6-7}
Method
& \makecell[c]{VerbMem\\$D_f$ ($\downarrow$)}
& \makecell[c]{KnowMem\\$D_f$ ($\downarrow$)}
& \makecell[c]{PrivLeak\\($\to 0$)}
& \makecell[c]{KnowMem\\$D_r$ ($\uparrow$)}
& \makecell[c]{Final\\Score ($\uparrow$)}
& \makecell[c]{Time (m)\\($\downarrow$)} \\
\midrule
\multicolumn{7}{c}{\textbf{MUSE News}} \\ \midrule

Original          &  58.29 &  62.93 & -98.71 & 54.31 & 40.50 & --- \\
Retain            &  20.75 &  33.32 &   0.00 & 53.79 & 67.88 & --- \\

\midrule

GA                & \redc 0.00 & \redc 0.00 & \orangec 20.14 & 0.00 & \yellowc 46.64 & 18.0 \\
GradDiff          & 4.85 & \yellowc 31.29 & 108.12 & 28.21 & 40.06 & 35.4 \\
Task Vector       & 77.42 & 58.76 & -100.00 & \redc 47.94 & 34.61 & 18.0 \\
NPO               & \yellowc 2.53 & 56.93 & 108.91 & \yellowc 37.58 & 40.73 & 42.4 \\
SimNPO            & \orangec 2.34 & 44.84 & \yellowc 72.93 & \orangec 39.65 & \orangec 49.81 & 36.3 \\
\ours{}           & 15.68 & \orangec 24.01 & \redc -3.80 & 26.37 & \redc 55.93 & \redc \textbf{0.3} \\

\midrule
\multicolumn{7}{c}{\textbf{MUSE Books}} \\
\midrule

Original          & 99.56 & 58.32 & -56.32 & 67.01 & 47.80 & --- \\
Retain            & 14.30 & 28.90 &   0.00 & 74.50 & 80.05 & --- \\

\midrule

GA                & \redc 0.00 & \redc 0.00 & \yellowc -24.07 & 0.00 & 45.99 & 26.7 \\
GradDiff          & \redc 0.00 & \redc 0.00 & -24.59 & 0.13 & 45.97 & 52.5 \\
Task Vector       & 99.31 & \orangec 35.55 & -83.78 & \orangec 62.55 & 44.84 & 26.7 \\
NPO               & \redc 0.00 & \redc 0.00 & -31.17 & 23.71 & \yellowc 56.66 & 62.5 \\
SimNPO            & \redc 0.00 & \redc 0.00 & \orangec -19.82 & \yellowc 48.27 & \orangec 70.83 & 53.8 \\
\ours{}           & \orangec 2.40 & \yellowc 36.89 & \redc -0.22 & \redc 65.56 & \redc 76.20 & \redc \textbf{0.3} \\

\bottomrule
\end{tabular}
\end{adjustbox}

\caption{Performance on MUSE News (LLaMA2-7B) and MUSE Books (ICLM-7B). Final Score is computed as $\frac{1}{2}(\mathrm{KnowMem}\,D_r + \mathrm{ForgetAvg})$, where $\mathrm{ForgetAvg}=\frac{(100-\mathrm{VerbMem}\,D_f)+(100-\mathrm{KnowMem}\,D_f)+(100-|\mathrm{PrivLeak}|)}{3}$. \ours{} attains the best Final Score among unlearning methods on both MUSE variants while being two orders of magnitude faster to apply. On MUSE Books it also reaches the highest retain utility (KnowMem $D_r$) while nearly eliminating verbatim memorization and driving privacy leakage to $\approx 0$.}
\label{tab:muse-unlearning}
\end{table}

\paragraph{\ours{} is quantization-robust.}
Unlearning can be undone simply by quantizing the model, which restores the supposedly
forgotten content~\citep{zhang2025quantfail}. We probe this on MUSE at full precision and
at 4-bit in Table~\ref{tab:quant}. A rise in forget memorization (VerbMem $D_f$, KnowMem $D_f$) means the unlearning
only hid the content. SimNPO fails sharply, its verbatim memorization climbing from near
zero back toward the original model on both corpora, with privacy leakage swinging back as
well. \ours{} does not recover, matching the gold Retrain model, which indicates that the
closed-form edit removes the targeted content rather than hiding it in low-magnitude
weights that rounding undoes. Retain utility (KnowMem $D_r$) drops at 4-bit for every
method including Original and Retrain.

\begin{table}[t]
\centering
\scriptsize
\setlength{\tabcolsep}{4pt}
\renewcommand{\arraystretch}{0.82}
\begin{adjustbox}{max width=\columnwidth}
\begin{tabular}{@{}l l ccc!{\vrule}cc@{}}
\toprule
& & \multicolumn{3}{c}{Unlearning Efficacy} & Utility & Summary\\
\cmidrule(r){3-5}\cmidrule(lr){6-6}\cmidrule(l){7-7}
Method & Prec. & \makecell[c]{VerbMem\\$D_f\,(\downarrow)$} & \makecell[c]{KnowMem\\$D_f\,(\downarrow)$}
& \makecell[c]{PrivLeak\\$(\to 0)$} & \makecell[c]{KnowMem\\$D_r\,(\uparrow)$} & \makecell[c]{Final\\Score ($\uparrow$)}\\
\midrule
\multicolumn{7}{c}{\textbf{MUSE News}}\\
\midrule
\multirow{2}{*}{Original} & full  & 58.29 & 62.93 & -98.71 & 54.31 & 40.50\\
                          & 4-bit & 45.8  & 55.6  & -99.8  & 48.5  & 40.7\\
\hdashline
\multirow{2}{*}{Retrain}  & full  & 20.75 & 33.32 &   0.00 & 53.79 & 67.88\\
                          & 4-bit & 19.7  & 36.5  &  -2.1  & 47.7  & 64.1\\
\hdashline
\multirow{2}{*}{SimNPO}   & full  &  2.34 & 44.84 &  72.93 & 39.65 & 49.81\\
                          & 4-bit & 38.7 & 48.2 & -99.8 & 47.4 & 42.6\\
\hdashline
\multirow{2}{*}{\ours{}}  & full  & 15.68 & 24.01 &  -3.80 & 26.37 & 55.93\\
                          & 4-bit & 13.5  & 22.8  &   0.5  & 20.9  & 54.3\\
\midrule
\multicolumn{7}{c}{\textbf{MUSE Books}}\\
\midrule
\multirow{2}{*}{Original} & full  & 99.56 & 58.32 & -56.32 & 67.01 & 47.80\\
                          & 4-bit & 94.5  & 36.2  & -60.4  & 50.6  & 43.5\\
\hdashline
\multirow{2}{*}{Retrain}  & full  & 14.30 & 28.90 &   0.00 & 74.50 & 80.05\\
                          & 4-bit & 14.1  & 24.5  &  -3.6  & 62.1  & 74.0\\
\hdashline
\multirow{2}{*}{SimNPO}   & full  &  0.00 &  0.00 & -19.82 & 48.27 & 70.83\\
                          & 4-bit & 72.5 & 34.0 & -59.4 & 52.4 & 48.6\\
\hdashline
\multirow{2}{*}{\ours{}}  & full  &  2.40 & 36.89 &  -0.22 & 65.56 & 76.20\\
                          & 4-bit &  2.9  & 29.4  &   2.1  & 45.1  & 66.8\\
\bottomrule
\end{tabular}
\end{adjustbox}
\caption{\textbf{Robustness to quantization on MUSE}, reproducing~\citep{zhang2025quantfail}.
A rise in forget memorization at 4-bit means the content was
only hidden: SimNPO recovers much of it, whereas \ours{} stays as low as the
gold Retrain model.}
\label{tab:quant}
\end{table}

\paragraph{Exact deletion auditability.}
We now use Theorem~\ref{thm:influence} as an evaluation tool, measuring how much each
forget example contributed to the edit. Writing $\gamma=\max_i\|\Delta P_i\|_F$ for the
largest single-example influence, the ratio $\gamma/\|P^\star\|_F$ is small on TOFU-10\%,
so no single example dominates the update, and it matches a brute-force re-solve.
Figure~\ref{fig:influence} shows a tight rather than heavy-tailed distribution, and
Appendix~\ref{app:proofs} confirms the $O(1/s)$ decay as the forget set grows. 

\begin{figure}[t]
\centering
\includegraphics[width=\columnwidth]{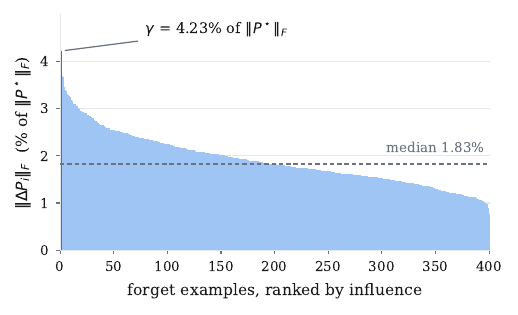}
\caption{\textbf{Per-example deletion influence on TOFU-10\%} (layer $15$, all $400$
forget examples, computed exactly via Theorem~\ref{thm:influence}). Each bar is the
influence $\|\Delta P_i\|_F$ of one forget example on the edit, as a percentage of
$\|P^\star\|_F$, sorted in decreasing order. The most influential example (dark)
defines $\gamma$. }
\label{fig:influence}
\end{figure}

\paragraph{Ablation: specificity weighting.}
We ask whether the specificity weight $\alpha$, rather than token suppression alone, is what balances forgetting against utility. We therefore rerun the best configuration of each suppression benchmark with $\alpha_j$ fixed to $1$ and compare it against the original weighting (Table~\ref{tab:ablations}). Disabling $\alpha$ lowers every score, and the drop is largest on the benchmarks where the forget and retain sets share the most vocabulary (MUSE News and TOFU-5\%), because uniform suppression also removes tokens that are needed on the retain set. The corollary is that the specificity weight, and not suppression alone, is the component that preserves utility.

\begin{table}[t]
\centering
\scriptsize
\setlength{\tabcolsep}{4pt}
\renewcommand{\arraystretch}{0.82}
\begin{adjustbox}{max width=\columnwidth}
\begin{tabular}{@{}l c cccc !{\vrule} c@{}}
\toprule
& Specificity & \multicolumn{4}{c}{Edited matrix} & \\
\cmidrule(lr){2-2}\cmidrule(lr){3-6}\cmidrule(l){7-7}
Benchmark & $\alpha$ off &
\texttt{o\_proj} & \texttt{up\_proj} & \texttt{q\_proj} & \texttt{v\_proj} &
\cellcolor[HTML]{EAF1FB}\ours{} \\
\midrule
TOFU-5\%   & 0.727 & \textbf{0.793} & 0.75 & 0.38 & 0.79 & \cellcolor[HTML]{EAF1FB}0.790 \\
TOFU-10\%  & 0.727 & 0.706 & 0.41 & 0.40 & 0.52 & \cellcolor[HTML]{EAF1FB}\textbf{0.747} \\
WMDP       & ---   & 0.62  & 0.46 & 0.46 & 0.48 & \cellcolor[HTML]{EAF1FB}\textbf{0.63} \\
MUSE-News  & 42.1  & 52.0  & 52.2 & 40.9 & 44.3 & \cellcolor[HTML]{EAF1FB}\textbf{55.5} \\
MUSE-Books & 72.4  & 68.48 & 62.9 & 52.4 & 69.5 & \cellcolor[HTML]{EAF1FB}\textbf{76.2} \\
\bottomrule
\end{tabular}
\end{adjustbox}
\caption{Two ablation studies of \ours{}, the \colorbox[HTML]{EAF1FB}{shaded} reference
column at right, which is the $\alpha$-on MLP \texttt{down\_proj} edit. \emph{Specificity:} disabling $\alpha$ (uniform suppression) lowers every
score. \emph{Edited matrix:}
each alternative matrix receives the strongest target it admits, and only \texttt{o\_proj}
on TOFU-5\% edges out \ours{}. }
\label{tab:ablations}
\end{table}

\paragraph{Ablation: PCA analysis.}
To visualize the action of the closed-form patch, we collect mean answer-position LM-head inputs $x$, compare $Wx$ with $(W+P)x$, and project each split separately onto three principal components. Figure~\ref{fig:pca-ablation} shows that the forget logits move substantially after the patch, while the retain logits remain nearly fixed.

\begin{figure}[t]
\centering
\includegraphics[width=\columnwidth]{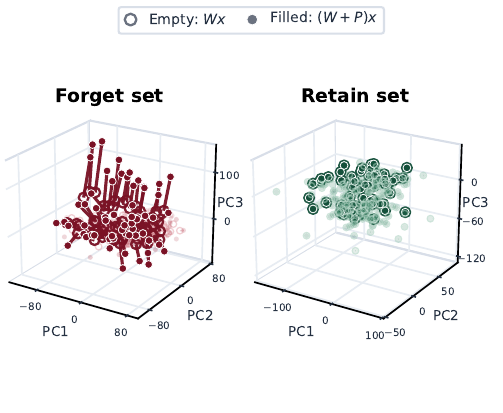}
\caption{3D PCA diagnostic of the LM-head patch. Hollow markers denote $Wx$, and filled markers denote $(W+P)x$. Forget points move visibly, whereas retain points nearly overlap. Only meaningful movements are shown in bold.}
\label{fig:pca-ablation}
\end{figure}

\paragraph{Ablation: which matrices are editable.}
Finally we apply the closed-form edit to each candidate matrix in turn, giving every matrix
the strongest target it admits (Table~\ref{tab:ablations}). Only the linear residual
writers realize the edit, as the Method discussion of editable matrices predicts, most clearly on WMDP where
every column shares a single target, and of the two we keep the MLP writer, which is better
or tied throughout and alone moves membership leakage on the long-form corpus. A further
ablation on \emph{which} layers to edit, comparing the attribution-selected band against
early, random, narrower, and wider alternatives, is reported in Appendix~\ref{app:additional-experiments}. We also provide hyperparameter sensitivity analysis in Appendix~\ref{app:hparam-sensitivity}.

\section{Conclusions}
We presented \ours{}, to our knowledge the first one-shot unlearning method that achieves and surpasses the unlearning performance of iterative, optimization based methods.
We argue that our method might pave the way to efficiently unlearning even larger models, where training based approaches are infeasible on limited hardware.
Our method is somewhat related in spirit to locate-and-edit methods for knowledge editing. 
It is an interesting avenue for future research to extend our method to the knowledge editing task.


\bibliography{aaai2027}

\appendix

\setcounter{secnumdepth}{2}
\numberwithin{table}{section}
\numberwithin{figure}{section}

\section{Additional Proofs}
\label{app:proofs}

\subsection{Proof of the Closed-Form Update}
\label{app:closed-form-proof}

We prove the closed-form update stated in Theorem~\ref{thm:closed-form}. First, we
record the invertibility of the matrix appearing in the solution.

\begin{apxlemma}
Let $X_r\in\mathbb{R}^{n\times r}$, $X_f\in\mathbb{R}^{n\times s}$, and let
$w_r,w_f,\mu>0$. Define
\[
\begin{aligned}
A
&=
\frac{w_r}{r}X_rX_r^\top
+
\frac{w_f}{s}X_fX_f^\top
+
\mu I_n .
\end{aligned}
\]
Then $A$ is symmetric positive definite. In particular, $A$ is invertible.
\end{apxlemma}

\begin{proof}
First, $A$ is symmetric because
\[
\begin{aligned}
(X_rX_r^\top)^\top &= X_rX_r^\top,\\
(X_fX_f^\top)^\top &= X_fX_f^\top,\\
I_n^\top &= I_n.
\end{aligned}
\]
Hence $A^\top=A$.

Now let $v\in\mathbb{R}^n$ be nonzero. Then
\[
\begin{aligned}
v^\top A v
&=
\frac{w_r}{r}\|X_r^\top v\|_2^2
+
\frac{w_f}{s}\|X_f^\top v\|_2^2
+
\mu\|v\|_2^2 .
\end{aligned}
\]
The first two terms are nonnegative, and the last term is strictly positive
since $\mu>0$ and $v\neq 0$. Therefore $v^\top A v>0$ for every nonzero $v$,
so $A$ is positive definite. Hence $A$ is invertible.
\end{proof}

\begin{apxtheorem}
Let $X_r\in\mathbb{R}^{n\times r}$, $X_f\in\mathbb{R}^{n\times s}$,
$D\in\mathbb{R}^{m\times s}$, and let $w_r,w_f,\mu>0$. Define
\[
\begin{aligned}
A
&=
\frac{w_r}{r}X_rX_r^\top
+
\frac{w_f}{s}X_fX_f^\top
+
\mu I_n
\in\mathbb{R}^{n\times n}.
\end{aligned}
\]
Then:
\begin{enumerate}
    \item The ridge-regularized objective
    \[
    \begin{aligned}
      \min_{P\in\mathbb{R}^{m\times n}}\;
      &
      \frac{w_r}{r}\,\|PX_r\|_F^2
      +
      \frac{w_f}{s}\,\|PX_f-D\|_F^2 \\
      &+
      \mu\,\|P\|_F^2
    \end{aligned}
    \]
    is minimized by
    \[
    P^\star
    =
    \frac{w_f}{s}\,
    D X_f^\top A^{-1}.
    \]
    Equivalently,
    \[
    \begin{aligned}
    P^\star
    &=
    \frac{w_f}{s}\,
    D X_f^\top
    \Big(
    \frac{w_r}{r}X_rX_r^\top \\
    &\qquad\qquad
    +
    \frac{w_f}{s}X_fX_f^\top
    +
    \mu I_n
    \Big)^{-1}.
    \end{aligned}
    \]

    \item The minimizer $P^\star$ is unique.
\end{enumerate}
\end{apxtheorem}

\begin{proof}
Let
\[
\begin{aligned}
\mathcal{L}(P)
&=
\frac{w_r}{r}\|PX_r\|_F^2
+
\frac{w_f}{s}\|PX_f-D\|_F^2
+
\mu\|P\|_F^2 .
\end{aligned}
\]
We rewrite each term using
\[
\|M\|_F^2=\operatorname{tr}(MM^\top).
\]
For the retain term,
\[
\begin{aligned}
\|PX_r\|_F^2
&=
\operatorname{tr}\big((PX_r)(PX_r)^\top\big)\\
&=
\operatorname{tr}(PX_rX_r^\top P^\top).
\end{aligned}
\]
For the forget term,
\[
\begin{aligned}
\|PX_f-D\|_F^2
=
\operatorname{tr}\big((PX_f-D)(PX_f-D)^\top\big).
\end{aligned}
\]
Expanding the product gives
\[
\begin{aligned}
&(PX_f-D)(PX_f-D)^\top\\
&\quad =
PX_fX_f^\top P^\top
-
PX_fD^\top
-
DX_f^\top P^\top
+
DD^\top .
\end{aligned}
\]
Therefore
\[
\begin{aligned}
\|PX_f-D\|_F^2
&=
\operatorname{tr}(PX_fX_f^\top P^\top)
-
\operatorname{tr}(PX_fD^\top)\\
&\quad
-
\operatorname{tr}(DX_f^\top P^\top)
+
\operatorname{tr}(DD^\top).
\end{aligned}
\]
Using $\operatorname{tr}(B)=\operatorname{tr}(B^\top)$, we have
\[
\begin{aligned}
\operatorname{tr}(DX_f^\top P^\top)
&=
\operatorname{tr}\big((PX_fD^\top)^\top\big)\\
&=
\operatorname{tr}(PX_fD^\top).
\end{aligned}
\]
Hence
\[
\begin{aligned}
\|PX_f-D\|_F^2
&=
\operatorname{tr}(PX_fX_f^\top P^\top)
-
2\operatorname{tr}(PX_fD^\top)\\
&\quad
+
\operatorname{tr}(DD^\top).
\end{aligned}
\]
Finally,
\[
\begin{aligned}
\|P\|_F^2
&=
\operatorname{tr}(PP^\top)
=
\operatorname{tr}(PI_nP^\top).
\end{aligned}
\]

Substituting these identities into $\mathcal{L}(P)$ gives
\[
\begin{aligned}
\mathcal{L}(P)
&=
\frac{w_r}{r}\operatorname{tr}(PX_rX_r^\top P^\top)
+
\frac{w_f}{s}\operatorname{tr}(PX_fX_f^\top P^\top)\\
&\quad
-
2\frac{w_f}{s}\operatorname{tr}(PX_fD^\top)
+
\frac{w_f}{s}\operatorname{tr}(DD^\top)\\
&\quad
+
\mu\operatorname{tr}(PI_nP^\top).
\end{aligned}
\]
Collecting the quadratic terms in $P$, and recalling that
\[
\begin{aligned}
A
&=
\frac{w_r}{r}X_rX_r^\top
+
\frac{w_f}{s}X_fX_f^\top
+
\mu I_n,
\end{aligned}
\]
we obtain
\[
\begin{aligned}
\mathcal{L}(P)
&= \operatorname{tr}(PAP^\top)
\\
&\quad
- \frac{2w_f}{s}\operatorname{tr}(PX_fD^\top)
\\
&\quad
+ \frac{w_f}{s}\operatorname{tr}(DD^\top).
\end{aligned}
\]
The final term is independent of $P$, so it does not affect the minimizer.

We now differentiate with respect to $P$. Since $A=A^\top$ by the lemma,
\[
\nabla_P \operatorname{tr}(PAP^\top)
=
2PA.
\]
Also,
\[
\operatorname{tr}(PX_fD^\top)
=
\operatorname{tr}(P(X_fD^\top)),
\]
so
\[
\begin{aligned}
\nabla_P \operatorname{tr}(PX_fD^\top)
&=
(X_fD^\top)^\top\\
&=
DX_f^\top.
\end{aligned}
\]
Therefore
\[
\begin{aligned}
\nabla_P\mathcal{L}(P)
&=
2PA
-
2\frac{w_f}{s}DX_f^\top.
\end{aligned}
\]

At any stationary point, the gradient must vanish:
\[
\begin{aligned}
2PA
-
2\frac{w_f}{s}DX_f^\top
=
0.
\end{aligned}
\]
Equivalently,
\[
PA
=
\frac{w_f}{s}DX_f^\top.
\]
By the lemma, $A$ is positive definite and therefore invertible. Multiplying on
the right by $A^{-1}$ gives
\[
P
=
\frac{w_f}{s}DX_f^\top A^{-1}.
\]
Thus
\[
P^\star
=
\frac{w_f}{s}DX_f^\top A^{-1}.
\]

It remains to show that this stationary point is the unique global minimizer.
Let $H\in\mathbb{R}^{m\times n}$ be any nonzero perturbation. Consider
\[
\mathcal{L}(P+H).
\]
Only the quadratic part determines strict convexity, and its second-order change is
\[
\operatorname{tr}(HAH^\top).
\]
Write the rows of $H$ as $h_1^\top,\ldots,h_m^\top$, where
$h_i\in\mathbb{R}^n$. Then
\[
\operatorname{tr}(HAH^\top)
=
\sum_{i=1}^m h_i^\top A h_i.
\]
Since $A$ is positive definite,
\[
h_i^\top A h_i\ge 0
\]
for every $i$, with equality only when $h_i=0$. Because $H\neq 0$, at least one
row $h_i$ is nonzero, and therefore
\[
\operatorname{tr}(HAH^\top)>0.
\]
Thus the quadratic part is strictly positive in every nonzero direction $H$.
Consequently, $\mathcal{L}$ is strictly convex in $P$. Hence the stationary point
$P^\star$ is the unique global minimizer.
\end{proof}

\subsection{Proof of the Deletion-Influence Result}
\label{app:deletion}

We prove the deletion-influence result stated in Theorem~\ref{thm:influence}.
Recall $P^\star=\tfrac{w_f}{s}\,D X_f^\top A^{-1}$ with
$A=\tfrac{w_r}{r}X_rX_r^\top+\tfrac{w_f}{s}X_fX_f^\top+\mu I_n$. Partition the forget
keys and target by example: $X_f=[\,C_1\,\cdots\,C_N\,]$, $D=[\,D_1\,\cdots\,D_N\,]$,
$C_i\in\mathbb{R}^{n\times c_i}$, $D_i\in\mathbb{R}^{m\times c_i}$. Here $c_i$ is the
number of answer tokens of example $i$, so $C_i$ holds the key vectors at those token
positions and $\sum_i c_i=s$. Throughout we abbreviate $\lambda=\tfrac{w_f}{s}$ and hold
$\lambda$ fixed at this value when an example is deleted. That is, we treat the objective
as weighting every forget token by the constant $\lambda$ rather than renormalizing by
the reduced number of forget tokens $s-c_i$, and all statements below are exact under
this convention.

\paragraph{The Sherman--Morrison--Woodbury identity.}
The proof rests on one standard matrix identity, which we state for completeness since it
does the essential work. For an invertible $A\in\mathbb{R}^{n\times n}$ and matrices
$U\in\mathbb{R}^{n\times c}$, $V\in\mathbb{R}^{c\times n}$ and invertible
$S\in\mathbb{R}^{c\times c}$, the Sherman--Morrison--Woodbury identity
\citep{hager1989updating} states
\begin{equation}
\label{eq:smw}
  (A+USV)^{-1}=A^{-1}-A^{-1}U\big(S^{-1}+VA^{-1}U\big)^{-1}VA^{-1},
\end{equation}
whenever the inverses involved exist. Its content is that a rank-$c$ perturbation of $A$
produces a rank-$c$ perturbation of $A^{-1}$, and that computing it costs only a
$c\times c$ inverse rather than a fresh $n\times n$ one. We use it in the
\emph{downdate} direction, $U=C_i$, $V=C_i^\top$ and $S=-\lambda I_{c_i}$, which is
exactly the perturbation caused by deleting one example. Lemma~\ref{lem:downdate} below
records that special case and verifies it directly, so the proof is self-contained.

\begin{apxlemma}[Rank-$c_i$ downdate]
\label{lem:downdate}
Let $A\succeq\mu I_n$ with $\mu>0$ and let $A_{-i}=A-\lambda C_iC_i^\top$ be positive
definite. Put
\[
  R_i=A^{-1}C_i\in\mathbb{R}^{n\times c_i},
\]
\[
  M_i=I_{c_i}-\lambda\,C_i^\top A^{-1}C_i\in\mathbb{R}^{c_i\times c_i}.
\]
Then $M_i$ is invertible and
$A_{-i}^{-1}=A^{-1}+\lambda\,R_i\,M_i^{-1}R_i^\top$.
\end{apxlemma}

\begin{proof}
\emph{$M_i$ is invertible.} Suppose $M_iv=0$ for some $v\in\mathbb{R}^{c_i}$, that is
$v=\lambda\,C_i^\top A^{-1}C_iv$, and set $u=A^{-1}C_iv$. Then $C_i^\top u=v/\lambda$ and
$u^\top Au=v^\top C_i^\top A^{-1}C_iv=\|v\|^2/\lambda$, so
\[
  u^\top A_{-i}u=u^\top Au-\lambda\|C_i^\top u\|^2
  =\tfrac{1}{\lambda}\|v\|^2-\tfrac{1}{\lambda}\|v\|^2=0 .
\]
Since $A_{-i}$ is positive definite this forces $u=0$, hence
$v=\lambda C_i^\top u=0$. So $M_i$ has trivial kernel and is invertible.

\emph{The formula.} Write $K_i=C_i^\top A^{-1}C_i$, so that $M_i=I_{c_i}-\lambda K_i$ by
definition. Multiply the claimed inverse by $A_{-i}$ and expand. Using $AR_i=C_i$ and
$C_i^\top A^{-1}=R_i^\top$, the product $(A-\lambda C_iC_i^\top)(A^{-1}+\lambda
R_iM_i^{-1}R_i^\top)$ has four terms,
\begin{multline*}
  I_n+\lambda C_iM_i^{-1}R_i^\top-\lambda C_iR_i^\top\\
  {}-\lambda^2 C_iK_iM_i^{-1}R_i^\top .
\end{multline*}
The last three share the left factor $C_i$ and the right factor $R_i^\top$, so they
collect into
\[
  \lambda C_i\big[(I_{c_i}-\lambda K_i)M_i^{-1}-I_{c_i}\big]R_i^\top ,
\]
and the bracket is $M_iM_i^{-1}-I_{c_i}=0$. The product is therefore $I_n$, which proves
the claim.
\end{proof}

\begin{apxtheorem}[Exact deletion influence]
For each $i$, $A_{-i}=A-\lambda C_iC_i^\top$ is symmetric positive definite, and
the update recomputed without example $i$ at fixed per-token weight $\lambda$,
$P^\star_{-i}=\lambda\,(DX_f^\top-D_iC_i^\top)A_{-i}^{-1}$, satisfies
\[
  \Delta P_i:=P^\star-P^\star_{-i}=L_iR_i^\top,\qquad R_i=A^{-1}C_i,
\]
\[
  L_i=\lambda\big(D_i-B_{-i}\,C_i\,M_i^{-1}\big),\quad
  M_i=I_{c_i}-\lambda\,C_i^\top A^{-1}C_i,
\]
where $B_{-i}=P^\star-\lambda\,D_iC_i^\top A^{-1}$. Hence
$\operatorname{rank}\Delta P_i\le c_i$, and $\Delta P_i$ is obtained from the existing
$A^{-1}$ by one $c_i\times c_i$ inverse, with no $n\times n$ reinversion.
\end{apxtheorem}

\begin{proof}
We proceed in four steps.

\emph{Step 1: what deleting example $i$ changes.} Because
$DX_f^\top=\sum_{k}D_kC_k^\top$ and $X_fX_f^\top=\sum_k C_kC_k^\top$ split as sums over
examples, removing example $i$ simply drops the $k=i$ term from each. The Gram matrix
becomes $A_{-i}=A-\lambda C_iC_i^\top$ and the numerator becomes
$\lambda(DX_f^\top-D_iC_i^\top)$, which we denote $B$. Both are rank-$c_i$ modifications
of quantities we have already computed.

\emph{Step 2: $A_{-i}$ stays positive definite.} The retain and ridge terms of $A$ are
untouched by the deletion and the remaining forget terms $\sum_{k\neq i}C_kC_k^\top$ are
positive semidefinite, so $A_{-i}\succeq\mu I_n$, exactly as in the Lemma of
Section~\ref{app:closed-form-proof}. In particular $A_{-i}$ is invertible, so
$P^\star_{-i}$ is well defined, and Lemma~\ref{lem:downdate} applies.

\emph{Step 3: invert the downdated Gram matrix.} By Lemma~\ref{lem:downdate},
\begin{equation}
\label{eq:step3}
  A_{-i}^{-1}=A^{-1}+\lambda\,R_iM_i^{-1}R_i^\top .
\end{equation}
This is the only place the deletion enters, and it costs one $c_i\times c_i$ inverse
against the cached $A^{-1}$.

\emph{Step 4: substitute and collect.} Insert \eqref{eq:step3} into
$P^\star_{-i}=BA_{-i}^{-1}$:
\[
  P^\star_{-i}=BA^{-1}+\lambda\,BR_iM_i^{-1}R_i^\top .
\]
We rewrite the two terms. For the first, $BA^{-1}=P^\star-\lambda D_iC_i^\top A^{-1}$ by
the definition of $B$, and the right-hand side is precisely $B_{-i}$, so
$BA^{-1}=B_{-i}$. For the second, $R_i=A^{-1}C_i$ gives
$BR_i=\big(BA^{-1}\big)C_i=B_{-i}C_i$. Hence
\[
  P^\star_{-i}=B_{-i}+\lambda\,B_{-i}C_iM_i^{-1}R_i^\top .
\]
Subtracting from $P^\star$ and using
$P^\star-B_{-i}=\lambda D_iC_i^\top A^{-1}=\lambda D_iR_i^\top$,
\begin{align*}
  \Delta P_i&=\lambda D_iR_i^\top-\lambda B_{-i}C_iM_i^{-1}R_i^\top\\
  &=\lambda\big(D_i-B_{-i}C_iM_i^{-1}\big)R_i^\top=L_iR_i^\top ,
\end{align*}
which is the stated factorization. The two terms have a direct reading: the first removes
example $i$'s own contribution to the target, and the second corrects for the fact that
deleting the example also shrinks the Gram matrix, which redistributes the edit across
the examples that remain. Finally $L_i\in\mathbb{R}^{m\times c_i}$ and
$R_i\in\mathbb{R}^{n\times c_i}$, so $\operatorname{rank}\Delta P_i\le c_i$, and every
quantity above involves $A$ only through the cached $A^{-1}$.
\end{proof}

\begin{remark}[First-order magnitude]
Dropping the $M_i^{-1}$ correction leaves the leading term
$\tfrac{w_f}{s}D_iC_i^\top A^{-1}$, so
$\|\Delta P_i\|_F\le\tfrac{w_f}{s}\|D_i\|_F\|C_i\|_2\,\|A^{-1}\|_2
\le \tfrac{w_f}{s\mu}\|D_i\|_F\|C_i\|_2$, using $A\succeq\mu I_n$. Influence therefore
decreases with forget-set size $s$ and ridge $\mu$.
\end{remark}

\paragraph{Influence under varying forget-set size and target.}
Recomputing the same diagnostic while varying the forget set confirms the $O(1/s)$
scaling of the Remark. On TOFU-10\% at layer $15$, the largest single-example influence
$\gamma$ relative to $\|P^\star\|_F$ falls from $0.173$ with $N=40$ forget examples, to
$0.079$ with $N=120$, and to $0.042$ with $N=400$. Substituting the
representation-corruption target used for WMDP, at the same $N=400$, gives $0.039$, so
the two unlearning targets distribute influence alike. All values are computed exactly
by the Woodbury downdate of Theorem~\ref{thm:influence}.

\paragraph{The bound is close to descriptive.}
On TOFU-10\% we computed $\|\Delta P_i\|_F$ for all $400$ forget examples and compared it
against the factors appearing in the Remark. Under the token-suppression target the
specificity mass of an example is $\|D_i\|_F/\beta=\sqrt{\sum_j\alpha_j^2}$, taken over
its answer positions. This quantity predicts the realized influence with Spearman
$\rho=0.85$ (Pearson $R^2=0.73$), whereas the answer-token count $c_i$ alone reaches only
$\rho=0.52$ ($R^2=0.30$). The original model's ROUGE-L recall and answer probability on
the same example give $|\rho|<0.1$, so influence is not a proxy for how strongly the
answer was memorized. The upper bound of the Remark therefore tracks the realized
influence closely, and the dominant factor is how much forget-specific content an example
contributes rather than its length or its memorization strength. Because $\alpha$ is
computed from token counts alone, this diagnostic requires no additional model
evaluation.

\section{Additional Experiments}
\label{app:additional-experiments}

This section reports additional experiments that do not fit in the main text: an
ablation on the choice of edited layers, and a sensitivity analysis over the
hyperparameters of the objective.

\subsection{Ablation: Choice of Edited Layers}
\label{app:abl-layers}
We next ask whether the attribution-selected layer band is genuinely useful, rather than
only its cardinality or late-layer location. We present ablation results in Table~\ref{tab:abl-layers-app}.
For each benchmark we keep the target, strength, retain weight, ridge
scale, and feature budgets fixed, and change only the edited \texttt{down\_proj} layers.
``Early'' edits the first $k$ layers, with $k$ matched to the selected band. ``Random'' is
a fixed noncontiguous same-cardinality control drawn once from a fixed seed. ``Single''
edits only the last layer of the selected band, ``Narrow'' edits its last two layers, and
``Wide'' extends the selected band toward earlier layers. Early layers destroy utility and
the single and narrow variants under-edit. The random control stays below the selected
band on all four benchmarks, but by an erratic margin, from $0.042$ on TOFU-10\% to an
outright collapse on TOFU-5\%, so same-cardinality alone does not recover the attribution
band. The selected band is therefore the most reliable choice, and it is the best variant
on every benchmark, including against the wider edit that spends more layers to get
there.

\begin{table}[t]
\centering
\scriptsize
\setlength{\tabcolsep}{2.4pt}
\renewcommand{\arraystretch}{1.15}
\resizebox{\columnwidth}{!}{%
\begin{tabular}{lcccccc}
\toprule
Benchmark & Early & Random & Single & Narrow &
\cellcolor[HTML]{EAF1FB}\makecell{Selected\\band} & Wide \\
\cmidrule(r){1-7}
TOFU-10\%  & \makecell{0.492\\$k=5$} & \makecell{0.705\\$k=5$} & \makecell{0.672\\$k=1$} & \makecell{0.716\\$k=2$} & \cellcolor[HTML]{EAF1FB}\makecell{0.747\\$k=5$} & \makecell{0.740\\$k=8$} \\
\cdashline{1-7}
TOFU-5\%   & \makecell{0.494\\$k=6$} & \makecell{0.492\\$k=6$} & \makecell{0.710\\$k=1$} & \makecell{0.723\\$k=2$} & \cellcolor[HTML]{EAF1FB}\makecell{0.790\\$k=6$} & \makecell{0.735\\$k=12$} \\
\cdashline{1-7}
MUSE-News  & \makecell{45.80\\$k=4$} & \makecell{49.11\\$k=4$} & \makecell{42.42\\$k=1$} & \makecell{44.17\\$k=2$} & \cellcolor[HTML]{EAF1FB}\makecell{55.49\\$k=4$} & \makecell{45.53\\$k=8$} \\
\cdashline{1-7}
MUSE-Books & \makecell{46.48\\$k=4$} & \makecell{71.60\\$k=4$} & \makecell{58.92\\$k=1$} & \makecell{64.66\\$k=2$} & \cellcolor[HTML]{EAF1FB}\makecell{76.27\\$k=4$} & \makecell{72.33\\$k=8$} \\
\bottomrule
\end{tabular}}
\caption{Layer-selection ablation with all non-layer hyperparameters fixed. Each cell
reports the score and number of edited MLP \texttt{down\_proj} layers $k$. }
\label{tab:abl-layers-app}
\end{table}

\subsection{Hyperparameter Sensitivity}
\label{app:hparam-sensitivity}

\paragraph{Only three hyperparameters are free.}
The objective carries three weights, $w_f$, $w_r$, and $\mu$, and the suppression target
adds an edit strength $\beta$. One of the four is redundant. Because we set the ridge
relative to the data scale, $\mu=\rho\,\bar{g}$ with $\bar{g}$ the mean diagonal entry of
$A$, rescaling $w_f$ and $w_r$ by a common factor multiplies both $A$ and the numerator
$\tfrac{w_f}{s}DX_f^\top$ of Theorem~\ref{thm:closed-form} by that factor, leaving the
update unchanged:
$$
P^\star(w_f,w_r,\rho)\;=\;P^\star\!\left(1,\;w_r/w_f,\;\rho\right).
$$
Only the ratio $w_r/w_f$ is a hyperparameter, and we fix the gauge $w_f=1$ throughout. We
confirmed this empirically rather than only asserting it: configurations related by the
identity above agree on every reported metric to within the seed-to-seed spread, with
residual differences consistent with rounding when the update is written into bfloat16
weights. The free hyperparameters are therefore the edit strength $\beta$, the retain
weight $w_r$, and the ridge scale $\rho$; the edit width $k$ is treated separately in
Section~\ref{app:abl-layers}.

\paragraph{Findings.}
Table~\ref{tab:hparam-sensitivity} shows that the three hyperparameters play qualitatively different roles.

\emph{Edit strength $\beta$ buys forgetting almost for free up to the tuned value.} At
$\beta=0$ the update is exactly $P=0$ and the model is unchanged. As $\beta$ grows to
$45$, forget efficacy rises from $0.199$ to $0.866$ while model utility stays at the
unedited level ($0.601\rightarrow0.600$): roughly five sixths of the attainable
forgetting costs no measurable utility. Utility only begins to pay past the tuned
$\beta=65$ and then falls sharply, reaching $0.127$ at $\beta=200$. 

\emph{The retain weight $w_r$ is the term that must be tuned.} It is the only axis whose
two ends both fail outright. Without retain anchoring ($w_r=0$) the edit destroys the
model, driving model utility to $0.000$ and retain ROUGE-L to $0.003$ while forgetting
saturates, which is the degenerate solution the retain term exists to prevent.
Over-weighting it ($w_r=1600$) suppresses the edit instead, and forget efficacy collapses
to $0.367$. Between these, the Final Score stays within $0.03$ of its maximum for
$w_r\in[50,200]$.

\emph{The ridge $\rho$ is a numerical safeguard rather than a tuning knob.} Across
$\rho\in[0.001,0.1]$, two orders of magnitude, the Final Score moves by $0.005$, which is
smaller than the $0.006$ seed spread. Performance degrades only once $\rho\geq0.3$, where
the penalty on $\|P\|_F$ starts to shrink the update itself. In practice $\rho$ can be
set to any small value that keeps $A$ comfortably conditioned.


\begin{table*}[t]
\centering
\small
\setlength{\tabcolsep}{4.5pt}
\renewcommand{\arraystretch}{1.12}
\begin{tabular}{lccccccccc}
\toprule
\multicolumn{10}{l}{\emph{Edit strength} $\beta$ \quad (at $w_r=100$, $\rho=0.03$)} \\
\cmidrule(r){1-10}
$\beta$   & 0 & 10 & 25 & 45 & \cellcolor[HTML]{EAF1FB}65 & 90 & 130 & 200 & 320 \\
MU        & 0.601 & 0.595 & 0.598 & 0.600 & \cellcolor[HTML]{EAF1FB}0.585 & 0.534 & 0.401 & 0.127 & 0.069 \\
F         & 0.199 & 0.601 & 0.817 & 0.866 & \cellcolor[HTML]{EAF1FB}0.906 & 0.937 & 0.957 & 0.975 & 0.984 \\
Final     & 0.400 & 0.598 & 0.708 & 0.733 & \cellcolor[HTML]{EAF1FB}\textbf{0.746} & 0.735 & 0.679 & 0.551 & 0.527 \\
\midrule
\multicolumn{10}{l}{\emph{Retain weight} $w_r$ \quad (at $\beta=65$, $\rho=0.03$)} \\
\cmidrule(r){1-10}
$w_r$     & 0 & 1 & 10 & 25 & 50 & \cellcolor[HTML]{EAF1FB}100 & 200 & 400 & 1600 \\
MU        & 0.000 & 0.000 & 0.036 & 0.295 & 0.508 & \cellcolor[HTML]{EAF1FB}0.585 & 0.605 & 0.597 & 0.598 \\
F         & 0.988 & 0.985 & 0.979 & 0.965 & 0.943 & \cellcolor[HTML]{EAF1FB}0.906 & 0.841 & 0.756 & 0.367 \\
Final     & 0.494 & 0.493 & 0.507 & 0.630 & 0.726 & \cellcolor[HTML]{EAF1FB}\textbf{0.746} & 0.723 & 0.677 & 0.483 \\
\midrule
\multicolumn{10}{l}{\emph{Ridge scale} $\rho$ \quad (at $\beta=65$, $w_r=100$; $\mu=\rho\,\bar{g}$)} \\
\cmidrule(r){1-10}
$\rho$    & 0.001 & 0.003 & 0.01 & \cellcolor[HTML]{EAF1FB}0.03 & 0.1 & 0.3 & 1.0 & & \\
MU        & 0.564 & 0.565 & 0.574 & \cellcolor[HTML]{EAF1FB}0.585 & 0.603 & 0.596 & 0.595 & & \\
F         & 0.918 & 0.916 & 0.912 & \cellcolor[HTML]{EAF1FB}0.906 & 0.885 & 0.838 & 0.758 & & \\
Final     & 0.741 & 0.741 & 0.743 & \cellcolor[HTML]{EAF1FB}\textbf{0.746} & 0.744 & 0.717 & 0.677 & & \\
\bottomrule
\end{tabular}
\caption{One-at-a-time hyperparameter sensitivity of \ours{} on TOFU-10\%. Each block
varies a single hyperparameter and holds the rest at the tuned configuration
(\colorbox[HTML]{EAF1FB}{shaded}). MU is model utility, F is forget efficacy
$\big((1-\text{Rouge-L})+(1-\text{Prob.})+(1-\text{Extr.\ Str.})\big)/3$, and Final is
their mean. Repeated seeds at the tuned setting span $0.014$ in MU, $0.002$ in F, and
$0.006$ in Final.}
\label{tab:hparam-sensitivity}
\end{table*}

\section{Unlearning on TOFU}
\label{app:tofu}

This section describes the evaluation metrics used in our TOFU experiments. We evaluate on the TOFU 5\% and 10\% forget splits.

\paragraph{Answer probability.}
For each example in the retain and forget splits, we measure how likely the model is to generate the reference answer conditioned on the corresponding question. Specifically, for a question $q$ and answer $a$, we compute the length-normalized conditional likelihood
\[
P(a \mid q)^{1/|a|},
\]
where $|a|$ denotes the number of tokens in the answer.

For the real authors and world facts subsets, each question is associated with five candidate answers: one correct answer $a_0$ and four perturbed, incorrect alternatives $\{\tilde{a}_1,\tilde{a}_2,\tilde{a}_3,\tilde{a}_4\}$. To quantify the model's preference for the correct answer, we use the normalized probability of the correct answer among all five candidates,
\[
\frac{P(a_0 \mid q)^{1/|a_0|}}
{P(a_0 \mid q)^{1/|a_0|} + \sum_{i=1}^{4} P(\tilde{a}_i \mid q)^{1/|\tilde{a}_i|}} .
\]

\paragraph{Truth ratio.}
The truth ratio measures the model's relative preference for incorrect answers over a correct paraphrased answer. Let $\hat{a}$ denote a paraphrase of the correct answer and let
$A=\{\tilde{a}_1,\tilde{a}_2,\ldots\}$ be the set of perturbed incorrect answers. We first compute the geometric mean of the length-normalized likelihoods assigned to the perturbed answers, and then divide this value by the length-normalized likelihood assigned to the paraphrased correct answer:
\[
R_{\text{truth}} =
\frac{
\left(\prod_{i=1}^{|A|} P(\tilde{a}_i \mid q)^{1/|\tilde{a}_i|}\right)^{1/|A|}
}{
P(\hat{a} \mid q)^{1/|\hat{a}|}
}.
\]
For the real authors and world facts subsets, paraphrased answers are not provided. In these cases, we use the original correct answer $a$ in place of $\hat{a}$.

As defined, a lower $R_{\text{truth}}$ indicates a stronger preference for the correct answer. The truth-ratio values reported in our tables, and the ones entering the model-utility aggregate, are therefore not $R_{\text{truth}}$ itself but the per-example transform $\max(0,\,1-R_{\text{truth}})$, averaged over the subset, so that higher reported values are better. On the forget split, where a truth ratio close to $1$ is desirable, the evaluation pipeline instead aggregates $\min(R_{\text{truth}},\,1/R_{\text{truth}})$, but our tables report the truth ratio only on the utility subsets.

\paragraph{ROUGE-L.}
For all TOFU subsets, we report ROUGE-L recall~\cite{lin2004rouge} between the ground-truth responses from the forget split and the model generations obtained after unlearning.

\paragraph{Extraction strength.}
We use extraction strength to evaluate how much of an answer that should have been forgotten the model can still complete once it is conditioned on the beginning of that answer. Following the OpenUnlearning implementation, the metric is computed with teacher forcing. The question and reference answer are passed through the model, and at every answer position we record the greedy (argmax) next-token prediction. For a question-answer pair $(q,a)$ with answer tokens $a_1,\dots,a_{|a|}$, let $k\ge 0$ be the smallest number of leading answer positions that must be discarded so that the greedy predictions agree with the reference tokens at all remaining positions. The suffix from position $k+1$ onward is thus the longest tail of the answer that the model reproduces correctly on its own, and the extraction strength is
\[
S_{\text{ext}}(q,a) = 1 - \frac{k}{|a|},
\]
the fraction of the answer recoverable in this way. Intuitively, a small $k$ means an attacker needs to supply only a short prefix of the answer before the model completes the rest verbatim. Higher extraction strength indicates that the target answer remains easier to elicit from the model, which corresponds to weaker unlearning. Lower extraction strength suggests stronger resistance to extraction.

\paragraph{Model utility.}
Finally, we report an aggregate model utility score. This score is computed as the harmonic mean of nine quantities: answer probability, truth ratio, and ROUGE-L recall, each evaluated on the retain, real authors, and world facts subsets. Higher model utility indicates better overall performance after unlearning.

\paragraph{Experimental details.}
Our method performs a single closed-form update per layer and involves no gradient-based training. For both TOFU settings we apply the token-suppression target to the MLP down-projection of a contiguous band of late layers, selected by the logit-lens attribution described in the main text. On TOFU-10\% (LLaMA-3.2-1B-Instruct) we edit layers $11$ through $15$ with edit strength $\beta=65$, retain weight $w_r=100$, and ridge scale $\rho=0.03$. On TOFU-5\% (LLaMA-2-7B-Chat) we edit layers $26$ through $31$ with $\beta=1000$, $w_r=300$, and $\rho=0.03$. In both cases the ridge coefficient is $\mu=\rho\,\bar{g}$, where $\bar{g}$ is the mean diagonal entry of the matrix $A$ in Theorem~\ref{thm:closed-form}, and the layers are edited sequentially with the keys recomputed after each edit.

\paragraph{Layer selection.}
For TOFU we use the token-suppression attribution score from the main text. For each
candidate layer, the score compares the layer's contribution to forget-set gold-token
logits against its contribution to retain-token logits. We then choose a contiguous
late-layer window with the highest average score for the chosen edit width $k$. This gives
layers $11$--$15$ for TOFU-10\% ($k=5$) and layers $26$--$31$ for TOFU-5\% ($k=6$). After
each layer edit, we recompute the keys before editing the next layer so that later updates
are computed on the current edited model.

\paragraph{Observed results.}
On TOFU-5\%, \ours{} obtains a Final Score of $0.79$, improving over the strongest baseline scores in our comparison while also achieving the highest MU score, $0.62$. The forget-set metrics are nearly saturated: $1$-Rouge-L is $0.95$, $1$-Prob. is $1.00$, and $1$-Extraction Strength is $0.97$. At the same time, retain-set quality remains substantially higher than for the strongest forgetting baselines; for example, retain ROUGE-L and retain probability are $0.83$ and $0.78$, respectively, compared with $0.54$ and $0.56$ for SimNPO. This indicates that the edit removes the selected fictitious-author facts without broadly suppressing the neighboring retain distribution.

TOFU-10\% is the more difficult TOFU setting because the forget set is larger and contains a broader set of author-specific associations. In this setting \ours{} obtains the best Final Score, $0.75$, and the best MU score, $0.60$. The method reaches $1$-Prob. of $1.00$ and $1$-Extraction Strength of $0.94$ on the forget set, while retaining the best truth-ratio scores on Real Authors, World Facts, and the retain split. The baseline pattern is instructive: IDKDPO reaches very strong forgetting but drops MU to $0.52$, whereas SimNPO keeps a more balanced profile but reaches a lower Final Score of $0.70$. \ours{} therefore sits at a better operating point on the forgetting--utility frontier. The measured update time is $0.5$ minutes on one H100, compared with $2.5$ minutes for SimNPO and $6.9$ minutes for AltPO and IDKDPO.

\paragraph{Full per-subset results.}
Tables~\ref{tab:tofu5-full} and~\ref{tab:tofu10-full} give the complete breakdown behind
the model-utility column of Table~\ref{tab:tofu}, reporting ROUGE-L, answer
probability, and truth ratio separately on the retain, Real Authors, and World Facts
subsets.

\begin{table*}[t]
\centering
\setlength{\tabcolsep}{3.5pt}
\renewcommand{\arraystretch}{1.15}
\begin{adjustbox}{max width=\textwidth}
\begin{tabular}{lcccccccccccc|ccc}
\toprule
& \multicolumn{3}{c}{Unlearning Efficacy} & \multicolumn{9}{c|}{Utility Preservation} & \multicolumn{3}{c}{Summary} \\
\cmidrule(lr){2-4} \cmidrule(lr){5-13} \cmidrule(lr){14-16}
Method &
\multicolumn{3}{c}{Forget Set} &
\multicolumn{3}{c}{Real Authors} &
\multicolumn{3}{c}{World Facts} &
\multicolumn{3}{c|}{Retain Set} &
\multirow{2}{*}{\centering MU ($\uparrow$)} &
\multirow{2}{*}{\centering Final Score ($\uparrow$)} &
\multirow{2}{*}{\centering Time (m) ($\downarrow$)}
\\
\cmidrule(lr){2-4} \cmidrule(lr){5-7} \cmidrule(lr){8-10} \cmidrule(lr){11-13}
& 1-Rouge-L$\uparrow$ & 1-Prob.$\uparrow$ & 1-Extr.\ Strength$\uparrow$ &
Rouge-L$\uparrow$ & Prob.$\uparrow$ & Truth ratio$\uparrow$ &
Rouge-L$\uparrow$ & Prob.$\uparrow$ & Truth ratio$\uparrow$ &
Rouge-L$\uparrow$ & Prob.$\uparrow$ & Truth ratio$\uparrow$ &
\\
\midrule
Original & 0.04 & 0.01 & 0.05 & 0.93 & 0.44 & 0.58 & 0.91 & 0.43 & 0.55 & 0.98 & 0.99 & 0.48 & 0.62 & 0.33 & --- \\
Retain   & 0.61 & 0.85 & 0.93 & 0.92 & 0.44 & 0.57 & 0.90 & 0.43 & 0.54 & 0.97 & 0.99 & 0.48 & 0.62 & 0.71 & --- \\ \hline

GradDiff & \redc 1.00 & \redc 1.00 & \orangec 0.96 & 0.59 & \redc 0.59 & \redc 0.81 & \orangec 0.88 & 0.46 & 0.59 & 0.42 & 0.49 & 0.48 & 0.56 & \orangec 0.77 & 2.4 \\
IDKDPO   & \orangec 0.98 & 0.40 & 0.85 & 0.65 & 0.48 & 0.63 & 0.82 & 0.44 & 0.55 & 0.55 & \redc 0.86 & \orangec 0.57 & \yellowc 0.57 & 0.66 & 3.3 \\
RKLD     & 0.69 & \yellowc 0.96 & \yellowc 0.92 & \redc 0.92 & 0.47 & 0.61 & \yellowc 0.87 & 0.47 & 0.58 & \yellowc 0.58 & 0.52 & \yellowc 0.56 & 0.56 & 0.71 & 2.9 \\
NPO      & 0.73 & 0.94 & 0.90 & \orangec 0.91 & \yellowc 0.50 & 0.62 & \redc 0.90 & \redc 0.50 & \orangec 0.61 & 0.47 & 0.51 & \orangec 0.57 & \yellowc 0.57 & 0.71 & 2.9 \\
SimNPO   & 0.74 & \orangec 0.97 & \yellowc 0.92 & \yellowc 0.90 & \yellowc 0.50 & \yellowc 0.64 & \redc 0.90 & \yellowc 0.48 & \yellowc 0.60 & 0.54 & 0.56 & \redc 0.58 & \orangec 0.58 & \yellowc 0.73 & 2.6 \\ \hdashline
ROME     & 0.82 & 0.69 & 0.90 & 0.62 & 0.44 & 0.56 & 0.68 & 0.44 & 0.58 & \orangec 0.67 & \orangec 0.84 & 0.48 & \yellowc 0.57 & 0.68 & \redc 0.6 \\
MEMIT    & 0.73 & 0.69 & 0.89 & 0.58 & \orangec 0.51 & \orangec 0.65 & 0.66 & \orangec 0.49 & \redc 0.62 & 0.45 & 0.50 & 0.43 & 0.53 & 0.65 & \orangec 0.7 \\ \hline
\ours{}  & \yellowc 0.95 & \redc 1.00 & \redc 0.97 & 0.72 & 0.49 & \orangec 0.65 & 0.85 & 0.46 & \redc 0.62 & \redc 0.83 & \yellowc 0.78 & 0.47 & \redc 0.62 & \redc 0.79 & \yellowc 1.8 \\

\bottomrule
\end{tabular}
\end{adjustbox}
\caption{Results on TOFU-5\% (LLaMA2-7B-Chat). Colors indicate rank (red: best, orange: second, yellow: third). Final Score is computed as $\frac{1}{2}\left(\mathrm{MU} + \frac{(1\text{-Rouge-L}) + (1\text{-Prob.}) + (1\text{-Extr.\ Strength})}{3}\right)$. Rows below the dashed rule are the locate-then-edit knowledge editors, which apply faster on this small forget set but reach a substantially worse trade-off. \ours{} achieves the best forgetting-utility trade-off.}

\label{tab:tofu5-full}
\end{table*}

\begin{table*}[t]
\centering
\setlength{\tabcolsep}{3.5pt}
\renewcommand{\arraystretch}{1.15}
\begin{adjustbox}{max width=0.95\textwidth}
\begin{tabular}{lcccccccccccc|ccc}
\toprule
& \multicolumn{3}{c}{Unlearning Efficacy} & \multicolumn{9}{c|}{Utility Preservation} & \multicolumn{3}{c}{Summary} \\
\cmidrule(lr){2-4} \cmidrule(lr){5-13} \cmidrule(lr){14-16}
Method &
\multicolumn{3}{c}{Forget Set} &
\multicolumn{3}{c}{Real Authors} &
\multicolumn{3}{c}{World Facts} &
\multicolumn{3}{c|}{Retain Set} &
\multirow{2}{*}{\centering MU ($\uparrow$)} &
\multirow{2}{*}{\centering Final Score ($\uparrow$)} &
\multirow{2}{*}{\centering Time (m) ($\downarrow$)}
\\
\cmidrule(lr){2-4} \cmidrule(lr){5-7} \cmidrule(lr){8-10} \cmidrule(lr){11-13}
& 1-Rouge-L$\uparrow$ & 1-Prob.$\uparrow$ & 1-Extr.\ Strength$\uparrow$ &
Rouge-L$\uparrow$ & Prob.$\uparrow$ & Truth ratio$\uparrow$ &
Rouge-L$\uparrow$ & Prob.$\uparrow$ & Truth ratio$\uparrow$ &
Rouge-L$\uparrow$ & Prob.$\uparrow$ & Truth ratio$\uparrow$
\\
\midrule
Original & 0.18 & 0.12 & 0.29 & 0.80 & 0.41 & 0.53 & 0.83 & 0.44 & 0.62 & 0.79 & 0.87 & 0.52 & 0.60 & 0.40 & --- \\
Retain   & 0.62 & 0.88 & 0.94 & 0.83 & 0.39 & 0.50 & 0.80 & 0.43 & 0.62 & 0.83 & 0.88 & 0.51 & 0.59 & 0.70 & --- \\ \hline

RMU      & 0.50 & 0.39 & 0.72 & \yellowc 0.75 & \yellowc 0.42 & 0.52 & \orangec 0.83 & 0.43 & 0.60 & 0.61 & 0.74 & \yellowc 0.51 & \yellowc 0.57 & 0.55 & \yellowc 0.6 \\
AltPO    & 0.66 & \yellowc 0.93 & \orangec 0.95 & \redc 0.78 & \yellowc 0.42 & \yellowc 0.54 & \yellowc 0.80 & 0.43 & \yellowc 0.61 & 0.61 & 0.75 & 0.47 & \yellowc 0.57 & \yellowc 0.71 & 6.9 \\
GradDiff & 0.42 & 0.35 & 0.67 & 0.73 & 0.40 & 0.53 & \redc 0.84 & 0.42 & \orangec 0.62 & \redc 0.79 & \redc 0.88 & \orangec 0.53 & \orangec 0.59 & 0.53 & 0.8 \\
IDKDPO   & \yellowc 0.87 & \redc 1.00 & \redc 0.96 & 0.41 & \yellowc 0.42 & 0.53 & 0.68 & 0.43 & 0.58 & 0.62 & 0.75 & 0.50 & 0.52 & \orangec 0.73 & 6.9 \\
IDKNLL   & \redc 0.98 & 0.45 & 0.74 & 0.70 & 0.39 & 0.49 & 0.73 & 0.43 & 0.57 & \yellowc 0.66 & \yellowc 0.78 & 0.49 & 0.55 & 0.64 & 0.8 \\
UNDIAL   & 0.69 & 0.82 & \redc 0.96 & 0.50 & 0.38 & 0.48 & 0.78 & 0.41 & 0.56 & 0.56 & 0.61 & 0.46 & 0.51 & 0.67 & 0.9 \\
NPO      & 0.61 & 0.71 & 0.91 & \orangec 0.76 & 0.41 & 0.52 & 0.79 & 0.43 & 0.60 & \yellowc 0.66 & \yellowc 0.78 & 0.50 & \yellowc 0.57 & 0.66 & 3.7 \\
SimNPO   & 0.65 & \orangec 0.94 & \yellowc 0.94 & \redc 0.78 & \yellowc 0.42 & 0.53 & \orangec 0.83 & \orangec 0.45 & \yellowc 0.61 & 0.56 & 0.71 & 0.48 & 0.56 & 0.70 & 2.5 \\ \hdashline
ROME     & \redc 0.98 & 0.37 & 0.60 & 0.62 & \yellowc 0.42 & \yellowc 0.54 & 0.76 & \yellowc 0.44 & \yellowc 0.61 & 0.45 & \orangec 0.79 & \yellowc 0.51 & 0.54 & 0.60 & \redc 0.4 \\
MEMIT    & \orangec 0.88 & 0.38 & 0.61 & 0.40 & \orangec 0.44 & \orangec 0.56 & 0.61 & \redc 0.47 & \yellowc 0.61 & 0.31 & 0.72 & 0.50 & 0.49 & 0.56 & \orangec 0.5 \\ \hline
\ours{}  & 0.78 & \redc 1.00 & \yellowc 0.94 & 0.63 & \redc 0.47 & \redc 0.63 & 0.71 & \redc 0.47 & \redc 0.68 & \orangec 0.71 & 0.77 & \redc 0.54 & \redc 0.60 & \redc 0.75 & \orangec 0.5 \\
\bottomrule
\end{tabular}
\end{adjustbox}
\caption{Results on TOFU-10\%. Colors indicate rank (red: best, orange: second, yellow: third). Final Score is computed as $\frac{1}{2}\left(\mathrm{MU} + \frac{(1\text{-Rouge-L}) + (1\text{-Prob.}) + (1\text{-Extr.\ Strength})}{3}\right)$. Rows below the dashed rule are the locate-then-edit knowledge editors. \ours{} achieves the best overall forgetting-utility trade-off, attaining the highest Final Score and model utility (MU) while avoiding the severe utility degradation seen in several stronger-forgetting baselines. ROME edits marginally faster on this small forget set, but at a far worse trade-off.}
\label{tab:tofu10-full}

\end{table*}

\section{Unlearning on MUSE}
\label{app:muse}

MUSE~\citep{shi2024muse} evaluates unlearning in long-form domains using memorization and privacy-leakage metrics computed on a forget split $\mathcal{D}_f$ and a retain split $\mathcal{D}_r$. We summarize the metrics used in our experiments below.

\paragraph{VerbMem.}
VerbMem (verbatim memorization) measures how much the model reproduces the forget text word for word. For each forget example the model is prompted to continue a prefix, and VerbMem scores the overlap between the continuation and the reference forget span, following the longest-common-subsequence criterion of~\citep{shi2024muse}. Lower VerbMem indicates less verbatim regurgitation of the forget content and is therefore preferred.

\paragraph{KnowMem.}
KnowMem (knowledge memorization) measures whether the model still expresses the underlying facts contained in the forget data, even when it does not reproduce them word for word. MUSE evaluates the model on question-style probes derived from the forget documents and checks whether the responses contain the target facts, using the automatic matching procedure of~\citep{shi2024muse}. We report KnowMem on both splits. On the forget split a lower value indicates better forgetting, whereas on the retain split a higher value indicates better preservation of non-forget knowledge.

\paragraph{PrivLeak.}
PrivLeak is a privacy-leakage proxy derived from the Min-$K\%$ Prob membership-inference attack \citep{shi2024detecting}. It measures how well an attacker can distinguish forget examples from a holdout set $\mathcal{D}_{\mathrm{holdout}}$ using model likelihood statistics. The holdout set is not the retain set. It is a disjoint set used as a non-member reference distribution for the membership test. PrivLeak is defined relative to a retraining baseline as
\begin{equation}
\begin{aligned}
\mathrm{PrivLeak}
&=
\frac{A_{\mathrm{unlearn}} - A_{\mathrm{retrain}}}{A_{\mathrm{retrain}}}
\times 100,
\\
A_m
&=
\mathrm{AUC}\!\left(
f_m,\,\mathcal{D}_f,\right.
\\[-0.2em]
&\qquad\left.
\mathcal{D}_{\mathrm{holdout}}
\right),
\\[-0.2em]
&\qquad
m \in \{\mathrm{unlearn},\mathrm{retrain}\},
\end{aligned}
\end{equation}
where $\mathrm{AUC}(\cdot)$ is the area under the ROC curve for separating samples of $\mathcal{D}_f$ and $\mathcal{D}_{\mathrm{holdout}}$ using Min-$K\%$ Prob features. A PrivLeak value closer to $0$ is better, indicating that the unlearned model approaches the retraining baseline in membership distinguishability.

\paragraph{Experimental details.}
For both MUSE corpora we apply the token-suppression target to the MLP down-projection of layers $28$ through $31$ of the released MUSE target model, with the suppression directions given by the logit-lens unembedding rows as in the main text. On MUSE News we use edit strength $\beta=370$ and retain weight $w_r=10$, and on MUSE Books we use $\beta=350$ and $w_r=80$, with ridge scale $\rho=0.03$ in both cases. The larger retain weight on Books reflects its more entangled forget and retain domains, in which a stronger retain anchor is needed to preserve closely related permissible knowledge.

\paragraph{Layer selection.}
For MUSE we use the same token-suppression layer-selection procedure as in TOFU. The
logit-lens attribution is computed on long-form forget and retain examples, and we select
the highest-scoring contiguous late-layer band with width $k=4$. This selects layers
$28$--$31$ for both MUSE News and MUSE Books. We keep the layer band fixed across the two
MUSE variants so that the comparison isolates the effect of the corpus and retain weight;
the Books run uses a larger retain weight because its retain examples are semantically
closer to the forget corpus.

\paragraph{Observed results.}
On MUSE News, \ours{} obtains the best Final Score among the unlearning methods, $55.93$, compared with $49.81$ for SimNPO. The most important change is privacy leakage: \ours{} drives PrivLeak to $-3.80$, close to the retraining reference value of $0$, whereas NPO and GradDiff leave large positive leakage values above $100$. Although GA obtains lower VerbMem and KnowMem on the forget set, it collapses retain KnowMem to $0.00$, so its Final Score remains lower. This illustrates the main MUSE trade-off: aggressively suppressing the forget documents is not sufficient if the method also destroys in-domain news utility.

On MUSE Books, the retain and forget distributions are more entangled because the retain material concerns the same fictional universe as the forget corpus. In this setting \ours{} obtains a Final Score of $76.20$, ahead of SimNPO at $70.83$, and preserves the highest retain KnowMem among unlearning methods, $65.56$. It also reduces VerbMem on the forget set from the original model's $99.56$ to $2.40$ and drives PrivLeak to $-0.22$, again close to the retraining target of $0$. Several baselines achieve zero forget-set memorization metrics, but they do so by sharply damaging retain utility; for example, GA and GradDiff reduce retain KnowMem to $0.00$ and $0.13$. The Books result therefore emphasizes why we tune the retain anchor more strongly on this benchmark.

The MUSE runtime gap is also large. The closed-form edit takes $0.3$ minutes for both MUSE News and MUSE Books, while the gradient baselines range from $18.0$ to $42.4$ minutes on News and from $26.7$ to $62.5$ minutes on Books under the same single-H100 timing convention. Thus the best MUSE Final Scores are obtained without an iterative fine-tuning run.

\section{Unlearning on WMDP}
\label{app:wmdp}

WMDP~\citep{wmdp} evaluates targeted capability suppression rather than memorization removal. The forget set is the WMDP-Bio subset of biosecurity-related multiple-choice questions, and the utility evaluation uses MMLU~\citep{hendrycks2020measuring} as a broad general-knowledge proxy. We follow the evaluation protocol of prior unlearning work on WMDP~\citep{simnpo}.

\paragraph{Question format.}
Each WMDP item is a multiple-choice question with a fixed set of answer options. We present the question and its options in a standard instruction prompt and score the model by the answer option to which it assigns the highest likelihood. We use the official WMDP-Bio split provided by the benchmark.

\paragraph{Forgetting and utility.}
We measure forgetting as the drop in WMDP-Bio accuracy after unlearning. Let $\mathrm{AccBio}$ denote the fraction of WMDP-Bio questions answered correctly. Following the main text, we report $1-\mathrm{AccBio}$ as the forgetting score, where larger values indicate stronger suppression of the hazardous slice. To quantify retained general capability we report overall MMLU accuracy under the same highest-likelihood decoding rule, where higher is better.

\paragraph{Experimental details.}
Because WMDP probes knowledge through multiple-choice accuracy rather than free generation, we use the representation-corruption target of the main text rather than token suppression. We apply the closed-form update to the MLP down-projection of a single mid layer, layer $8$ of LLaMA-3-8B-Instruct, with corruption strength $c=45$, retain weight $w_r=1000$, and ridge scale $\rho=0.03$. The forget keys are collected from the WMDP-Bio forget corpus and the retain keys from a generic WikiText corpus, which we found essential for preserving MMLU. As discussed in the main text, we select the mid layer following the established convention for representation-level unlearning on WMDP.

\paragraph{Layer selection.}
For WMDP we use a representation-corruption edit rather than token suppression, so the
logit-lens token attribution used for TOFU and MUSE is not the selection criterion. We edit
a single mid layer, layer $8$, matching the layer range commonly used for
representation-level WMDP unlearning. The edit is applied to the MLP
\texttt{down\_proj}, and the retain keys are taken from WikiText to anchor general
language-model behavior while the WMDP-Bio keys receive the corruption target.

\paragraph{Observed results.}
On WMDP-Bio, \ours{} obtains a Final Score of $0.63$, the highest among the compared unlearning methods. SimNPO has the strongest raw forgetting score, with $1-\mathrm{AccBio}=0.75$, but its MMLU accuracy is $0.44$. \ours{} gives slightly less suppression, $1-\mathrm{AccBio}=0.71$, while preserving substantially higher MMLU accuracy, $0.55$. This difference is the reason \ours{} has the best overall trade-off despite not maximizing the forgetting metric alone.

The single-layer closed-form edit reduces WMDP-Bio accuracy while keeping broad utility
closer to the original model. The update takes $0.2$ minutes, compared with roughly
$20$--$24$ minutes for the gradient baselines.

\section{Unlearning on ZsRE}
\label{app:zsre}

ZsRE~\citep{levy2017zeroshot} is a few-shot factual editing benchmark. We use it in the
reverse direction from TOFU, MUSE, and WMDP: instead of asking whether knowledge editors
perform well on broad unlearning benchmarks, we ask whether \ours{} remains competitive in
the fact-editing harness used by ZeroUnlearn~\citep{lin2026zerounlearn}. This makes the
comparison favorable to the locate-then-edit baselines, since ROME, MEMIT, AlphaEdit, and
ZeroUnlearn are designed for this style of localized fact intervention.

\paragraph{Task format.}
Each ZsRE record contains a requested rewrite with a subject, a prompt template, and the
original answer, together with paraphrased prompts and neighborhood prompts. In a standard
editing setup the method would replace the original answer with a new target. For
unlearning, we instead treat the original answer as the content to remove: the edit should
make the model stop predicting that answer on the original prompt and on paraphrases,
while preserving answers to the neighborhood prompts.

\paragraph{Metrics.}
We follow the released ZeroUnlearn evaluation harness. \emph{Efficacy} is the percentage
of rewrite prompts on which the edited model still predicts the original answer tokens, so
lower values indicate stronger unlearning. \emph{Generalization} is the same measurement
on paraphrased prompts and tests whether the deletion transfers beyond the exact wording
of the request. \emph{Specificity} is accuracy on neighborhood prompts that should remain
unchanged, so higher values indicate better locality preservation. The main text reports
mean and standard deviation across runs.

\paragraph{Experimental details.}
We use the same few-shot protocol as the ZeroUnlearn harness: each run unlearns $50$ ZsRE
facts and uses $1000$ disjoint facts as the retain set. Results are averaged over ten
seeds. The forget keys use the fact-editing \texttt{subject\_last} convention, i.e., one
key is collected at the final subject token rather than at every answer token. This
matches the ROME/MEMIT/ZeroUnlearn convention and gives the closed-form system a retain
anchor in the same feature space as the editor baselines.

For ZsRE, the best \ours{} configuration is head-only. We set the MLP suppression strength
to zero, edit the untied LM head with the token-suppression target, use specificity
reweighting over the original-answer tokens, and keep the retain anchor on the disjoint
retain facts. For both evaluated models we use head strength $\beta_{\mathrm{head}}=20$,
retain weight $w_r=30$, ridge scale $\rho=0.03$, and no additional WikiText retain
anchor. The layer lists in the configuration are retained only for compatibility with the
shared harness; because $\beta_{\mathrm{MLP}}=0$, the MLP updates are exactly zero.

\paragraph{Observed results.}
Table~\ref{tab:zsre} shows that \ours{} suppresses the original answers much
more strongly than the locate-then-edit methods while preserving locality. The improvement
appears on both direct rewrite prompts and paraphrases, indicating that the edit is not
simply breaking one prompt surface form. Specificity remains at the level of the strongest
editing baselines on the smaller model and is highest in the table on the larger model.
This supports the main conclusion that the closed-form suppression edit is not only a
broad benchmark method: it also works inside the few-shot factual-unlearning setting for
which the editor baselines were tuned.

\section{Benchmark and Evaluation Summary}
\label{app:summary}

Table~\ref{tab:benchmarks} summarizes the target model and evaluation metrics used
for each benchmark. We separate metrics that measure forgetting on the forget set
from metrics that measure utility preservation on retain or held-out evaluation
sets. Arrows indicate the desired direction after unlearning.

\begin{table*}[t]
\centering
\footnotesize
\setlength{\tabcolsep}{4pt}
\renewcommand{\arraystretch}{1.15}
\begin{adjustbox}{max width=\textwidth}
\begin{tabular}{@{} l l l l @{}}
\toprule
\textbf{Benchmark}
&
\textbf{Target Model}
&
\textbf{Unlearning Effectiveness}
&
\textbf{Utility Preservation}
\\
\midrule

TOFU
&
\makecell[l]{LLaMA-2-7B-Chat\\LLaMA-3.2-1B-Instruct}
&
\begin{tabular}[t]{@{}l@{\hspace{4pt}}c@{}}
Probability on $\mathcal{D}_f$ & $\downarrow$\\
ROUGE-L on $\mathcal{D}_f$ & $\downarrow$\\
Extraction strength on $\mathcal{D}_f$ & $\downarrow$
\end{tabular}
&
\begin{tabular}[t]{@{}l@{\hspace{4pt}}c@{}}
Model utility & $\uparrow$\\
Probability on $\mathcal{D}_r$, Real Authors, World Facts & $\uparrow$\\
ROUGE-L on $\mathcal{D}_r$, Real Authors, World Facts & $\uparrow$\\
Truth ratio on $\mathcal{D}_r$, Real Authors, World Facts & $\uparrow$
\end{tabular}
\\
\midrule

MUSE
&
\makecell[l]{LLaMA-2-7B\\ICLM-7B}
&
\begin{tabular}[t]{@{}l@{\hspace{4pt}}c@{}}
KnowMem on $\mathcal{D}_f$ & $\downarrow$\\
VerbMem on $\mathcal{D}_f$ & $\downarrow$\\
PrivLeak & $\to 0$
\end{tabular}
&
\begin{tabular}[t]{@{}l@{\hspace{4pt}}c@{}}
KnowMem on $\mathcal{D}_r$ & $\uparrow$
\end{tabular}
\\
\midrule

WMDP
&
\makecell[l]{Llama-3-8B-Instruct}
&
\begin{tabular}[t]{@{}l@{\hspace{4pt}}c@{}}
Accuracy on WMDP-Bio & $\downarrow$
\end{tabular}
&
\begin{tabular}[t]{@{}l@{\hspace{4pt}}c@{}}
Accuracy on MMLU & $\uparrow$
\end{tabular}
\\
\midrule

ZsRE
&
\makecell[l]{LLaMA-3.2-3B\\LLaMA-3.1-8B}
&
\begin{tabular}[t]{@{}l@{\hspace{4pt}}c@{}}
Rewrite orig.-answer acc. & $\downarrow$\\
Paraphrase orig.-answer acc. & $\downarrow$
\end{tabular}
&
\begin{tabular}[t]{@{}l@{\hspace{4pt}}c@{}}
Neighborhood acc. & $\uparrow$
\end{tabular}
\\
\bottomrule
\end{tabular}
\end{adjustbox}
\caption{Benchmark summary. For each benchmark, we report the target model and the metrics used to evaluate unlearning effectiveness and utility preservation.}
\label{tab:benchmarks}
\end{table*}

\paragraph{Runtime protocol.}
All runtime values reported in the experimental tables are wall-clock update times measured
on a single NVIDIA H100. The timing excludes model loading, tokenizer loading, dataset
preprocessing, and the final benchmark evaluation. For gradient-based baselines, the timer
starts immediately before the unlearning optimization loop and ends after the final
optimizer step. Thus the reported time includes the forward and backward passes, optimizer
updates, and any gradient-checkpointing overhead used by the released recipe. For
\ours{}, the timer starts immediately before the closed-form edit procedure and ends after
the edited weights have been written back to the model. This includes collecting the
forget and retain activations used as keys, constructing the ridge-regularized normal
equations, solving for the update, and applying the update to each selected layer. It does
not include the subsequent evaluation pass used to compute TOFU, MUSE, WMDP, ZsRE, or MMLU
metrics.
The baseline budgets behind these numbers follow the released reference recipe of each
benchmark: ten epochs at effective batch $32$ on TOFU-5\%, ten epochs at effective batch
$64$ with gradient checkpointing on MUSE, and $500$ steps at effective batch $4$ on WMDP,
the last following the released OPTML recipes for NPO, GradDiff and
SimNPO~\citep{simnpo}. In every case \ours{} performs a single closed-form edit.

\begin{table*}[t]
\centering
\footnotesize
\setlength{\tabcolsep}{5pt}
\renewcommand{\arraystretch}{1.15}
\begin{adjustbox}{max width=\textwidth}
\begin{tabular}{@{} l l l l l @{}}
\toprule
\textbf{Benchmark}
&
\textbf{Target model}
&
\textbf{Edited module}
&
\textbf{Edited layers}
&
\textbf{Selection rule}
\\
\midrule
TOFU-10\%
&
LLaMA-3.2-1B-Instruct
&
MLP \texttt{down\_proj}
&
$11$--$15$ ($k=5$)
&
Contiguous late-layer band with the largest average logit-lens attribution score.
\\
TOFU-5\%
&
LLaMA-2-7B-Chat
&
MLP \texttt{down\_proj}
&
$26$--$31$ ($k=6$)
&
Contiguous late-layer band with the largest average logit-lens attribution score.
\\
MUSE News
&
LLaMA-2-7B
&
MLP \texttt{down\_proj}
&
$28$--$31$ ($k=4$)
&
Contiguous late-layer band with the largest average logit-lens attribution score.
\\
MUSE Books
&
ICLM-7B
&
MLP \texttt{down\_proj}
&
$28$--$31$ ($k=4$)
&
Same MUSE layer band; retain weight is increased for the more entangled Books setting.
\\
WMDP-Bio
&
LLaMA-3-8B-Instruct
&
MLP \texttt{down\_proj}
&
$8$ ($k=1$)
&
Single mid-layer representation edit, following the WMDP representation-unlearning setup.
\\
ZsRE
&
\makecell[l]{LLaMA-3.2-3B-Instruct\\LLaMA-3.1-8B-Instruct}
&
LM head
&
---
&
Head-only closed-form token suppression using the fact-editing
\texttt{subject\_last} key convention.
\\
\bottomrule
\end{tabular}
\end{adjustbox}
\caption{Layer-selection summary. Layer numbers use the model indexing convention in our implementation. For token-suppression benchmarks, $k$ is the edit width and the selected band maximizes the average logit-lens attribution score over candidate late layers.}
\label{tab:layer-selection}
\end{table*}

\section{Additional Reproducibility Details}
\label{app:reproducibility}

\paragraph{Code and evaluation pipeline.}
Our implementation consists of scripts for collecting forget and retain activations,
constructing the closed-form update, applying the edited weights, and launching the
standard benchmark evaluations. The code, the configuration file of every reported
setting, and the command lines needed to reproduce the reported experiments are available
at \url{https://github.com/Batorskq/GROM} under the MIT license. All reported benchmark
scores are computed with the
OpenUnlearning evaluation pipeline~\citep{dorna2025openunlearning}, except for ZsRE,
which is evaluated in the ZeroUnlearn harness~\citep{lin2026zerounlearn}. We use the
benchmark datasets and target checkpoints provided by TOFU, MUSE, WMDP, ZsRE, and the
corresponding released evaluation harnesses rather than introducing any new dataset.

\paragraph{Software and hardware.}
Experiments were run in a Linux HPC environment on a single NVIDIA H100 GPU. The software
environment used Python~3.11, PyTorch~2.4.1, Transformers~4.51.3, and an editable
installation of OpenUnlearning. Model loading and evaluation use the Hugging Face
Transformers stack with CUDA acceleration. Unless otherwise specified by the corresponding
benchmark recipe, model computations use the precision of the released target checkpoint
and the OpenUnlearning evaluation configuration.

\paragraph{Randomness and number of runs.}
The closed-form edit is deterministic once the forget and retain examples, target
directions, edited layers, and scalar hyperparameters are fixed. For the WMDP
representation-corruption target, the random direction is generated once with seed $0$ and
then held fixed. We use the default seed $0$ in the OpenUnlearning evaluation.

\paragraph{Use of LLM assistance.}
Large language models were used solely for editorial and auxiliary support, including
improving clarity, grammar, and presentation, and providing assistance with implementation
code. All core technical contributions, experimental design decisions, analyses,
interpretations, and final research judgments were made by the authors.

\end{document}